\newif\iffinal
\finaltrue
\newif\ifarxiv
\arxivtrue
\newif\iffuture
\futurefalse
\newif\iflong
\longfalse

\ifarxiv
\documentclass{article}
 \usepackage[nonatbib,preprint]{neurips_2020}
\usepackage{amsmath,amsthm}
\usepackage{microtype} 
\usepackage[usenames,dvipsname s,svgnames,table]{xcolor}
\else

\documentclass[anon]{colt2020} 
\usepackage{times}
   \usepackage[english]{babel}        \usepackage{appendix}

\newcounter{assumption}
\newenvironment{assumption}[1][]{\refstepcounter{assumption}\par\medskip
   \noindent \textbf{Assumption~\theassumption. #1} \rmfamily}{\medskip}
  
   \newcounter{observation}[chapter]

    \newcounter{lemma}[chapter]
\renewenvironment{lemma}[1][]{\refstepcounter{lemma}\par\medskip
   \noindent \textbf{Lemma~\thelemma. #1} \rmfamily}{\medskip}
   
    \newcounter{definition}[chapter]
\renewenvironment{definition}[1][]{\refstepcounter{lemma}\par\medskip
   \noindent \textbf{Definition~\thelemma. #1} \rmfamily}{\medskip}

\fi

\newcommand{\sref}[2]{\hyperref[#2]{#1 \ref*{#2}}}
\usepackage{amssymb}
\usepackage{subcaption}
\usepackage{graphicx}
\usepackage{float}
\usepackage{placeins}
\usepackage{paralist}
\usepackage[utf8]{inputenc}   
\usepackage[T1]{fontenc} 

\usepackage{paralist}
\usepackage{url}

\ifarxiv
\usepackage{caption} 

\usepackage{aliascnt}
\def\NewTheorem#1#2{%
  \newaliascnt{#1}{theorem}
  \newtheorem{#1}[#1]{#2}
  \aliascntresetthe{#1}
  \expandafter\def\csname #1autorefname\endcsname{#2}
}
 \newtheorem{theorem}{Theorem}[section]
\NewTheorem{lemma}{Lemma}
\NewTheorem{corollary}{Corollary}
\NewTheorem{conjecture}{Conjecture}
\NewTheorem{observation}{Observation}
\NewTheorem{definition}{Definition}
\NewTheorem{assumption}{Assumption}

\NewTheorem{proposition}{Proposition}
\NewTheorem{remark}{Remark}
\NewTheorem{claim}{Claim}
\NewTheorem{fact}{Fact}

\else

\fi

\DeclareMathOperator{\lmax}{\ell_{max}}

\newcommand\eps{\ensuremath{\varepsilon}}

\usepackage{color}

\definecolor{darkred}{rgb}{0.5,0,0}
\definecolor{lightblue}{rgb}{0,0.4,0.8}
\definecolor{darkgreen}{rgb}{0,0.5,0}

\renewcommand{\epsilon}{\varepsilon}

\newcommand\E[1]{\mathbb{E}\left[\,#1\,\right]}
\renewcommand{\Pr}[1]{\mathbb{P}\left[\,#1\,\right]}

\iffinal

\else
\definecolor{blanchedalmond}{rgb}{1.0, 0.92, 0.9}
	\definecolor{Dandelion}{rgb}{0.95, 0.95, 0.7}
\pagecolor{blanchedalmond}

\fi

\newcommand{\nnote}[1]{\iffinal \else {\color{purple}[N: #1 ]} \fi}
\newcommand{\snote}[1]{\iffinal  \else \iffuture {\color{blue}[ #1 ]}\fi\fi}

\newcommand{\corr}[1]{{#1}}

\usepackage[colorlinks=true, linkcolor=blue, urlcolor=blue, citecolor=darkgreen]{hyperref}
\usepackage{authblk}
\date{\today} 
\title{Learning Hierarchically-Structured Concepts}
\date{}

\author[1]{Nancy Lynch}
\author[2]{Frederik Mallmann-Trenn}
 \affil[1]{Massachusetts Institute of Technology, Cambridge, Massachusetts, USA}
 \affil[2]{King's College London, London, England}

\begin{document}
\maketitle 
\begin{abstract}
We use a recently developed synchronous Spiking Neural Network (SNN) model to study the problem of learning hierarchically-structured concepts.
We introduce an abstract data model that describes simple hierarchical concepts. 
We define a feed-forward layered SNN model, with learning modeled using Oja's local learning rule, a well known biologically-plausible rule for adjusting synapse weights.
We define what it means for such a network to recognize hierarchical concepts; our notion of recognition is robust, in that it tolerates a bounded amount of noise.

Then, we present a learning algorithm by which a layered network may learn to recognize hierarchical concepts according to our robust definition.
We analyze correctness and performance rigorously; the amount of time required to learn each concept, after learning all of the sub-concepts, is approximately 
$O\left(\frac{1}{\eta k} \left(\lmax \log(k) + \frac{1}{\epsilon} \right) + b \log(k)\right)$, where $k$ is the number of sub-concepts per concept, $\lmax$ is the maximum hierarchical depth, $\eta$ is the learning rate, $\epsilon$ describes the amount of uncertainty allowed in robust recognition, and $b$ describes the amount of weight decrease for "irrelevant" edges.\snote{I decided to be explicit about b.}
An interesting feature of this algorithm is that it allows the network to learn sub-concepts in a highly interleaved manner.
This algorithm assumes that the concepts are presented in a noise-free way; we also extend these results to accommodate noise in the learning process.
Finally, we give a simple lower bound saying that, in order to recognize concepts with hierarchical depth two with noise-tolerance, a neural network should have at least two layers.

The results in this paper represent first steps in the theoretical study of hierarchical concepts using SNNs.  The cases studied here are basic, but they suggest many directions for extensions to more elaborate and realistic cases.
\end{abstract}

{\bf Keywords:} 
Hierarchical Concepts, 
Representing Hierarchical Concepts,
Recognizing Hierarchical Concepts,
Learning Hierarchical Concepts, 
Spiking Neural Networks,
Brain-Inspired Algorithms

\setcounter{tocdepth}{1}

\section{Introduction}

We are interested in the general problem of \emph{how concepts that have structure are represented in the brain}.  What do these representations look like? How are they learned, and how do the concepts get recognized after they are learned?
We draw inspiration from recent experimental research on computer vision in convolutional neural networks (CNNs) by Zeiler and Fergus~\cite{Zeiler} and Zhou, et al.~\cite{Zhou19}.  
This research shows that CNNs learn to represent structure in visual concepts:  lower layers of the network represent basic concepts and higher layers represent successively higher-level concepts.
%
%
This observation is consistent with neuroscience research, which indicates that visual processing in mammalian brains is performed in a hierarchical way, starting from primitive notions such as position, light level, etc., and building toward complex objects; see, e.g.,~\cite{Hubel59, Hubel62, Felleman91}.
More generally, we consider the thesis that \emph{the structure that is naturally present in real-world concepts get mirrored in their brain representations, in some natural way that facilitates both learning and recognition}.

We approach this problem using ideas and techniques from theoretical computer science, distributed computing theory, and in particular, from recent work by Lynch, et al. on synchronous Spiking Neural Networks (SNNs) ~\cite{CamITCS,LMP19,CamComposition,SuCL19,Renaming}.
These papers began the development of an algorithmic theory of SNNs, developing formal foundations, and using them to study problems of attention and focus, neural representation, and short-term learning.
Here we continue that general development, by initiating the study of long-term learning within the same framework.

We focus here on learning hierarchically-structured concepts.
We capture these formally in terms of abstract \emph{concept hierarchies}, in which concepts are built from lower-level concepts, which in turn are built from still-lower-level concepts, etc.
Such structure is natural, e.g., for physical objects that are learned and recognized during human or computer visual processing.
An example of such a hierarchy might be the following model of a \emph{human}: 
A human consists of a \emph{body}, a \emph{head}, a \emph{left leg}, a \emph{right leg}, a \emph{left arm}, and a \emph{right arm}. 
Each of these concepts may consist of more concepts, allowing us to model a human to an arbitrary degree of granularity. 
Most concepts in the real world have additional structure, e.g., arms and legs are positioned symmetrically; however, we ignore such information for now and assume simply that each concept consists of sub-concepts.
For this initial theoretical study, we make some additional simplifications:  we fix a maximum level $\lmax$ for concept hierarchies, we assume that all non-primitive concepts have the same number $k$ of "child concepts", and we assume that our concept hierarchies are trees, i.e., there is no overlap in the composition of different concepts at the same level of a hierarchy.
We expect that these assumptions can be removed or weakened, but it seems useful to start with the simplest case. 


This paper demonstrates theoretically, in terms of simple hierarchies, how hierarchically-structured data can be represented, learned, and recognized in feed-forward layered Spiking Neural Networks.  
Specifically, we provide formal definitions for \emph{concept hierarchies} and \emph{layered neural networks}.  
We define precisely what it means for a layered neural network to \emph{recognize} a particular concept in a concept hierarchy.
Our notion of recognition is \emph{robust}:  a concept is required to be recognized if the input is close to the ideal concept, and is required not to be recognized if the input is far from the ideal.
We also define what it means for a layered neural network to \emph{learn to recognize} a concept hierarchy, according to our robust definition of recognition.

Next, we present two simple, efficient algorithms (layered networks) that learn to recognize concept hierarchies; the first assumes reliability during the learning process, whereas the second tolerates a bounded amount of noise.
An example of such learning is shown in \autoref{fig:main}.
We also provide a preliminary lower bound, saying that, in order to robustly recognize concepts with hierarchical depth $2$, a neural network should have at least $2$ layers.  We discuss possible extensions of this bound to concepts with larger depth.
We end with many directions for extending this work.
\setlength{\belowcaptionskip}{-10pt}
\setlength{\abovecaptionskip}{-9pt}
\begin{figure}
\centering
\begin{minipage}{.7\textwidth}
  \centering
 \includegraphics[page=4,width=0.99\textwidth]{plan}
\end{minipage}%
\begin{minipage}{.3\textwidth}
  \centering
    \includegraphics[width=0.95\textwidth]{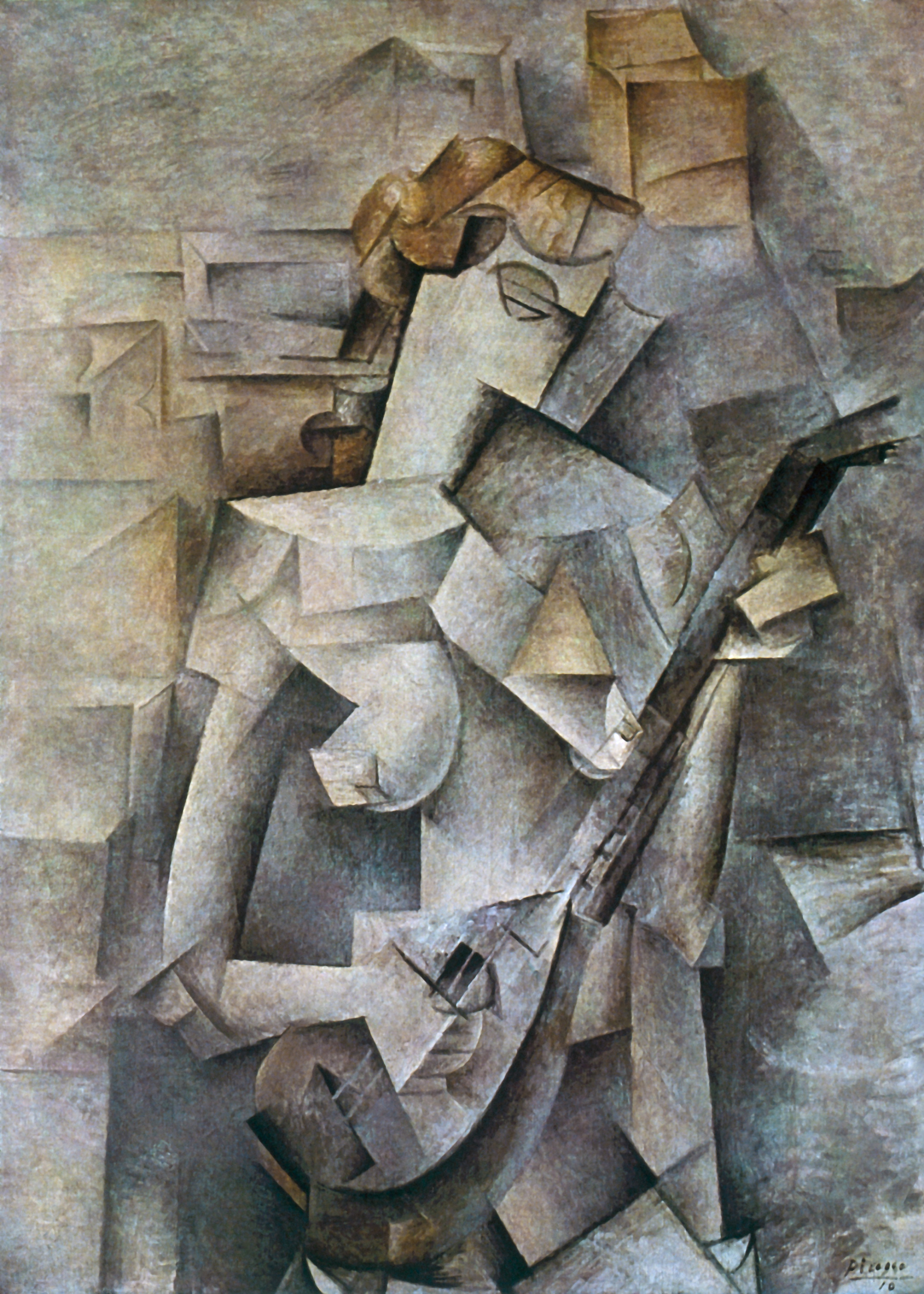}
\end{minipage}%
\caption{The leftmost figure shows the concept \emph{human}, which consists of two sub-concepts, and so on. The second figure shows a network that has "learned" the concept "human" in the sense that, when the neurons representing the basic parts \emph{eyes, mouth, arms, legs} are excited,
  then exactly one neuron $u$ on the top layer will fire.  
  Neuron $u$ should also fire when "most" of the basic parts are excited, and $u$ should not fire when few of the basic parts are excited. For example, the painting ``Girl  with a Mandolin'' by Picasso\protect\footnotemark should cause $u$ to fire despite the lack of a mouth and legs. 
  The network accomplishes this by strengthening relevant synapses (bold edges) and weakening others (thin edges). 
  }
  \label{fig:main}
\end{figure}

\footnotetext{\copyright \ Universal Color Slide Company.}

\emph{Note:}  We view this work as the first step in a general project to produce a theory for how logical concepts are represented, and learned, in the brain. 
Our general approach is to start with the simplest case, working out basic definitions, algorithms, and limitations for that case, and then to extend in many directions, step-by-step.  
We think such a stepwise approach will be effective in developing the theory. 
In addition, we hope that this first step, besides being of interest on its own, will provide a useful blueprint for later extensions.

\paragraph{In more detail:}
We describe our data model in \autoref{sec:datamodel}.
We assume a fixed maximum number $\lmax$ of levels in our concept hierarchies. Each concept hierarchy $\mathcal C$ has a fixed set $C$ of concepts, organized into levels $\ell$, $0 \leq \ell \leq \lmax$.
These are chosen from some universal set $D$ of \emph{concepts}.
Each concept at each level $\ell$, $1 \leq \ell \leq \lmax$ has precisely $k$ children, which are level $\ell-1$ concepts.
We assume here that each concept hierarchy is a tree, that is, there is no overlap among the sets of children of different concepts.
\iffuture
\footnote{
   Considering overlap introduces a complication in the method of selecting new representations; later we discuss some ways to get around the problem.\nnote{I assume we will say something about this, maybe give a couple of suggestions in the future work section about possibilities for an additional mechanism for avoiding re-selecting the same neuron.}}
   \fi
Each individual concept hierarchy represents the concepts and child relationships that arise in a particular execution of the network (or lifetime of an organism).
However, the chosen concepts and their relationships may be different in different concept hierarchies. Again we note that these assumptions are a considerable simplification of reality, but we regard them as a good starting point.

Next, in \autoref{sec:networkmodel}, we define a synchronous Spiking Neural Network model\footnote{
A word about our use of the Spiking Neural Network terminology:
Our model here is simpler than typical SNN models, in that neuron actions depend just on the previous state and not on a longer history.  In some of our prior work, such as~\cite{SuCL19}, we use a more elaborate version of the model in which neurons actions can depend on bounded history. This is useful for capturing aspects of neuron processing such as accumulating potential.  In future extensions of the present work, we expect to use such elaborations.
We use the SNN terminology here in an attempt to keep the terminology consistent across our papers.
}, 
derived from the one in~\cite{CamITCS,CamComposition}, but with additional structure to support learning.
Namely, the new model incorporates edge weights (representing synapse strengths) into neuron states; this provides a convenient way to describe how those weights change during learning.
We model learning using \emph{Oja's rule}, a biologically-inspired rule that can be regarded as a mathematical formalization of Hebbian learning~\cite{kempter1999hebbian}. 
Oja's rule was first introduced in~\cite{Oja}, and has since received considerable attention due its connections with dimensionality reduction; see, for example, \cite{oja1992principal, foldiak1989adaptive}.
Although there is no direct experimental evidence yet that Oja's precise rule is used in the brain, its core characteristics such as long-term potentiation, long-term depression, and normalization are known to occur in brain networks, and have been studied thoroughly (e.g., ~\cite{artola1993long,artola1990different}).
Interestingly, to the best of our knowledge, Oja's rule has so far been studied only in "flat" settings, where the network has only one layer. 
Moreover, previous work (e.g., \cite{Oja}) has allowed the learning parameter $\eta$ to be time-dependent, in order to achieve convergence.  
In this paper, we consider the multilayer setting, and we show convergence with a fixed learning rate.

In \autoref{sec:probstatement}, we present our definitions for the robust recognition and noise-free learning problems. 
Thus, we define how an SNN represents a concept hierarchy; here we use the simplifying assumption that each concept is represented by just one neuron.
We define what it means for an SNN to correctly recognize a concept hierarchy, including situations in which the network is required to recognize a concept $c$ and situations where it is required not to do so.
In particular, if a sufficiently large fraction $r_2$ of the children of concept $c$ are recognized, then $c$ should be recognized, whereas if fewer than a smaller fraction $r_1$ of the children of $c$ are recognized, then $c$ should not be recognized.
We also define what it means for an SNN to learn to recognize a concept hierarchy, in the noise-free setting. 

Then, in \autoref{sec:algorithms}, we present algorithms that allow a network, starting from a default configuration, to recognize and to learn the concepts in a particular concept hierarchy.
Our algorithms are efficient, in terms of network size and running time.
In particular, a network with max layer $\lmax$ suffices to recognize a concept hierarchy with max level $\lmax$.
Recognition happens within a very short time, proportional to the number of layers in the network.
For learning, our algorithm converges reasonably quickly to a configuration that supports robust recognition. Our convergence time bound result for noise-free learning is \autoref{thm:noisefreelearning}.
\snote{Later, give stability guarantees.}
Our algorithms require the examples to be shown several times and in a constrained order: roughly speaking, we require the network to "learn" the children of a concept $c$ first, before examples of $c$ are shown.
Thus, in our running example, we require enough examples of "head", "body", etc. to be able to learn those concepts before the network sees them all together as "human".
Except for this constraint, concepts may be shown in an arbitrarily interleaved manner.
In \autoref{sec:noisy}, we adapt our problem definitions and learning algorithm to a setting where the examples presented may be perturbed by noise.
The modified algorithm still works, but now convergence requires the network to see more examples, compared to the noise-free case, as we show in \autoref{thm:noisylearning}.
The detailed analysis needed to prove Theorems~\ref{thm:noisefreelearning} and~\ref{thm:noisylearning} appears in Sections~\ref{sec:noisefreeanalysis} and~\ref{sec:noisyanalysis}, respectively.

\iffuture
In \autoref{as:WTA}, we outline how this engagement assumption could be implemented, using a \emph{Winner-Take-All} strategy~\cite{CamITCS}.

\fi

Once we see that a network with max layer $\lmax$ can easily learn and recognize any concept hierarchy with max level $\lmax$, it is natural to ask whether $\lmax$ layers are actually necessary.  Certainly these networks yield natural and efficient representations, but it is still interesting to ask the theoretical question of whether shallower networks could accomplish the same thing.  
In \autoref{sec:lowerbounds}, we give a preliminary lower bound result, showing that a two-layer concept hierarchy requires a two-layer network in order to solve the noisy recognition problem.  We also discuss the possibility of extending this result to more levels and layers.
%


In summary, this paper is intended to show, using theoretical techniques, how structured concepts can be represented, recognized, and learned in biologically plausible neural networks.  
We give fundamental definitions and algorithms for particular types of concept hierarchies and networks.
This represents a first step towards a theory of representation and learning for hierarchically-structured concepts in SNNs; it opens up many follow-on questions, which we discuss in \autoref{sec:futurework}.

\paragraph{Related work:}
Immediate inspiration from this work came from experimental computer vision research on "network dissection" by Zhou, et al.~\cite{Zhou19}.
%
This work describes experiments that show that unsupervised learning of visual concepts in deep convolutional neural networks results in "disentangled" representations.
These include neural representations, not just for the main concepts of interest, but also for their components and sub-components, etc., throughout a concept hierarchy.
As in this paper, they consider individual neurons as representations for individual concepts.
They find that the representations that arise are generally arranged in layers so that more primitive concepts (colors, textures,...) appear at lower layers whereas more complex concepts (parts, objects, scenes) appear at higher layers.
Earlier work by Zeiler and Fergus~\cite{Zeiler} made similar observations.
%
%
As we described earlier, this work is consistent with neuroscience research, which indicates that visual processing in mammalian brains is performed hierarchically~\cite{Hubel59, Hubel62, Felleman91}.
Some of this work indicates that the network includes feedback edges in addition to forward edges; the function of the feedback edges seems to be to solidify representations of lower-level objects based on context~\cite{Hupe98,markov2014anatomy}.
While we do not yet address feedback edges in this paper, that is one of our main intended future directions.

{
Brain-like hierarchical models have been studied before (e.g., \cite{riesenhuber1999hierarchical} and \cite{synfire}).
The authors of \cite{riesenhuber1999hierarchical} propose a model consisting of different kinds of  cells to model image recognition in the brain.
Another biologically-motivated line of research concerns synfire chains, which are essentially a feed-forward network of neurons. These networks are a predecessor of spiking neural networks (SNNs).
An interesting work in this field is \cite{synfire}, which studies a hierarchical organization of synfire chains.

}

The SNN 
model~\cite{maass1996computational,maass1997networks,gerstner2002spiking,izhikevich2004model,habenschuss2013stochastic}, upon which all of our neural algorithms research is based, is a model for neural computation that balances biological plausibility with theoretical tractability.
Our work is influenced by research of Maass et al.~\cite{maass1997networks,maass1999neural,maass2000computational} on the computational power of SNNs, and by that of Valiant~\cite{valiant2000circuits,valiant2000neuroidal,valiant2005memorization,valiant2012hippocampus} on learning in the \emph{neuroidal model} of brain computation.
Recent research by Papadimitriou, et al.~\cite{PapadimitriouVempala-pmlr15,PapadimitriouVMCM19,LegensteinMPV18,PapadimitriouVempala-itcs19} on problems of learning and association of concepts is another source of inspiration.

Oja's learning rule~\cite{Oja,oja1992principal}. is a biologically plausible local rule for adjusting synapse weights during learning.   
As mentioned earlier, to the best of our knowledge, Oja's rule has so far been studied only in single-layered networks and with time-dependent learning rates (\cite{Oja, oja1992principal, foldiak1989adaptive}.
Other related learning rules include Hebbian variants~\cite{Hebb49,LowelS92} or BCM learning~\cite{bienenstock1982theory}.

The learning algorithms in this paper utilize a \emph{Winner-Take-All} sub-network~\cite{lazzaro1988winner,yuille1989winner,thorpe1990spike,coultrip1992cortical,maass2000computational,wang2003k,oster2006spiking,LynchMP17a}, to help in selecting which neurons to engage in learning.  Winner-Take-All is an important primitive in neural computation that is used to model visual attention and competitive learning. 

Work by Mhaskar et al.~\cite{Poggio16}
is related to ours in that they also consider embedding a tree-structured concept hierarchy in a layered network.
They also prove results saying that deep neural networks are better than shallow networks at representing a deep concept hierarchy,
However, their concept hierarchies differ mathematically from ours, since they are formalized as compositional functions.
\iflong
Also, their notion of representation corresponds to function approximation, and their proofs are based on approximation theory, rather than the limitations imposed by requiring robustness in recognizing hierarchical concepts.
Other results along the same lines appear in~\cite{Matus}.
\else
Also, their notion of representation is different, corresponding to function approximation, and their proofs are based on approximation theory.
\fi
Other related work appears in papers by Knoblauch and collaborators, e.g., \cite{rewC,rewD,rewF}.
These papers describe experimental work involving hierarchical concepts that are more general than ours (e.g., allowing overlap), networks that are more general (e.g., allowing feedback), and more robust types of representations (cell assemblies).  They present this work in the context of an integrated robot system combining processing of visual and language input, decisions, and action).
For us, this provides good inspiration for future theoretical work.

\iflong
There is also an interesting connection to circuit complexity (e.g., \cite{kopparty2012certifying}) with respect to the question of how many layers are required to solve the recognition problem (\autoref{sec:prob-recog}). The models studied are slightly different as neurons have the power of threshold gates.

Nonetheless, understanding the trade-off between the number of layers and the number of neurons per layer would be a very interesting question for future work.

\fi
\paragraph{Acknowledgments:}
We thank Brabeeba Wang for helpful conversations and suggestions. we also thank  an anonymous referee for much constructive feedback, and many suggestions for interesting extensions.
The authors were supported in part by NSF Award Numbers CCF-1810758, CCF-0939370, CCF-1461559, and CCF-2003830.
\section{Data Model}
\label{sec:datamodel}

In this section, we define an abstract notion of a \emph{concept hierarchy}, which represents all the concepts that arise in some particular "lifetime" of an organism, together with hierarchical relationships between them.  As noted above, our definition is restricted to tree-structured hierarchical relationships; extensions are left for future work.
We follow this with a definition for the notion of \emph{support}, which indicates which lowest-level concepts are sufficient to trigger the recognition of higher-level concepts.

\subsection{Preliminaries}

We begin by defining some general notation.  First, we fix four constants:
\begin{itemize}
    \item $\lmax$, a positive integer, representing the maximum level number for the concepts we consider.  
    \item $n$, a positive integer, representing the total number of lowest-level concepts.
    \item $k$, a positive integer, representing the number of top-level concepts in any concept hierarchy, and also the number of sub-concepts for each concept that is not at the lowest level.\footnote{Assuming the same number $k$ throughout is a simplification of what would be needed for applications; it should be easy to generalize this.}
    \item $r_1, r_2$, reals in $[0,1]$ with $r_1 \leq r_2$; these represent thresholds for noisy recognition.
\end{itemize}

We assume a predetermined universal set $D$ of \emph{concepts}, partitioned into disjoint sets $D_{\ell}, 0 \leq \ell \leq \lmax$.
We refer to any particular concept $c \in D_{\ell}$ as a \emph{level} $\ell$ \emph{concept}, and write $level(c) = \ell$.
Here, $D_0$ represents the most basic concepts and $D_{\lmax}$ the highest-level concepts.
We assume that $|D_0| = n$.

\subsection{Concept hierarchies}

A \emph{concept hierarchy} $\mathcal C$ consists of a subset $C$ of $D$, together with a $children$ function.  For each $\ell$, $0 \leq \ell \leq \lmax$, we define $C_{\ell}$ to be $C \cap D_{\ell}$, that is, the set of level $\ell$ concepts in $\mathcal C$.
For each concept $c \in C_{\ell}$, $1 \leq \ell \leq \ell_{max}$, we designate a nonempty set $children(c) \subseteq C_{\ell-1}$.
We call each $c' \in children(c)$ a \emph{child} of $c$.
We require the following three properties.
\begin{enumerate}
\item
$|C_{\lmax}| = k$.
\item
For any $c \in C_{\ell}$, where $1 \leq \ell \leq \lmax$, we have that $|children(c)| = k$; that is, the degree of any internal node in the concept hierarchy is exactly $k$.
\item
For any two distinct concepts $c$ and $c'$ in $C_{\ell}$, where $1 \leq \ell \leq \lmax$, we have that $children(c) \cap children(c') = \emptyset$; that is, the sets of children of different concepts at the same level are disjoint.
\end{enumerate}

It follows that ${\mathcal C}$ is a forest with $k$ roots and height $\lmax$.  Also, for any $\ell, 0 \leq \ell \leq \lmax$, $|C_{\ell}| = k^{\lmax - \ell + 1}$.
Note that our notion of concept hierarchies is quite restrictive, in that we allow no overlap between the sets of children of different concepts. Allowing overlap is an important next direction for future work.

We extend the $children$ notation recursively by defining a concept $c'$ to be a $descendant$ of a concept $c$ if either $c' = c$, or $c'$ is a child of a descendant of $c$.
%
We write $descendants(c)$ for the set of descendants of $c$.
Let $leaves(c) = descendants(c) \cap C_0$, that is, all the level 0 descendants of $c$.

\subsection{Support}
\label{sec:support}

Now we give a key definition that indicates which lowest-level concepts should be sufficient to trigger recognition of higher-level concepts.  

We fix a particular concept hierarchy $\mathcal C$, with its concept set $C$ partitioned into $C_0,\ldots,C_{\lmax}$.
For any given subset $B$ of the general set $D_0$ of level $0$ concepts, and any real number $r \in [0,1]$, we define a set $supported_r(B)$ of concepts in $C$.
This represents the set of concepts $c \in C$, at all levels, that have enough of their leaves present in $B$ to support recognition of $c$. 
The notion of "enough" here is defined recursively, based on having an $r$-fraction of children supported at every level.

\FloatBarrier
\begin{definition}[\textbf{Supported}]
\label{def:support}
Given $B \subseteq D_0$, define the following sets of concepts at all levels, recursively:
\begin{enumerate}
\item
$B_0 = B \cap C_0$. That is, we restrict attention to just the level $0$ concepts in $C$.
\item
$B_1$ is the set of all concepts $c \in C_1$ such that $|children(c) \cap B_0|  \geq r k$.  That is, we consider the level $1$ concepts in $C$ for which at least an $r$-fraction of their children appear in $B_0$.
\item
For $2 \leq \ell \leq \ell_{max}$, $B_{\ell}$ is the set of all concepts $c \in C_{\ell}$ such that $|children(c) \cap B_{\ell - 1}|  \geq r k$.  That is, we consider the level $\ell$ concepts in $C$ for which at least an $r$-fraction of their children appear in $B_{\ell-1}$.
 \end{enumerate}
Define $supported_r(B)$ to be $\bigcup_{0 \leq \ell \leq \lmax} B_{\ell}$.  
We sometimes also write $supported_r(B,\ell)$ for $B_{\ell}$.
\end{definition}
\FloatBarrier


\begin{figure}[ht!]
\centering
\includegraphics[width=0.6\textwidth]{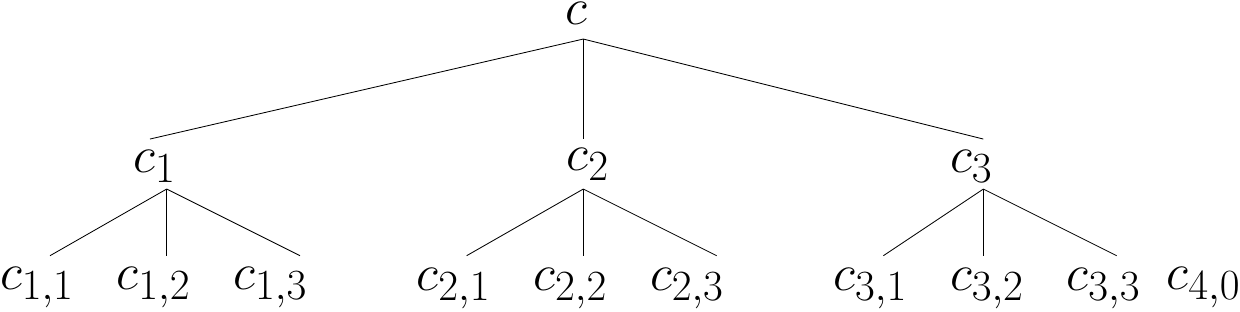}
\vspace{1cm}
\caption{This example illustrates the $supported_r(B)$ definition, with $k = 3$ and $r = \frac{2}{3}$.  We depict just a single level $2$ concept $c$ with children $c_1, c_2, c_3$ and grandchildren $c_{1,1}, c_{1,2}, c_{1,3}, c_{2,1}, c_{2,2}, c_{2,3}, c_{3,1}, c_{3,2}, c_{3,3}$.  The set $B$ consists of concepts $c_{1,1}$, $c_{1,2}$, $c_{3,1}, c_{3,3}$ plus an "extra" concept $c_{4,0}$ that is not a descendant of $c$. Then $B_0 = \{c_{1,1}, c_{1,2}, c_{3,1}, c_{3,3}\}$, $B_1 = \{c_1, c_3\}$, and $B_2 = \{c\}$.}
\end{figure}

\vspace{.2cm}
The special case $r=1$ is important as it corresponds to a "noise-free" notion of support, in which all the leaves of a concept must be present.  That is:
\begin{lemma}
\label{lem: support-special-case}
For any $B \subseteq D_0$, $supported_1(B)$ is the set of all concepts $c \in C$ (at all levels) such that $leaves(c) \subseteq B$.
\end{lemma}

%
\section{Network Model}
\label{sec:networkmodel}
\FloatBarrier

In this section, we define our network model.
We first describe the network structure, then the individual neurons, and finally the operation of the overall network.  

\subsection{Preliminaries}

We introduce four constants:
\begin{itemize}
     \item
     $\ell'_{max}$, a positive integer, representing the maximum number of a layer in the network.
     \item
     $n$, a positive integer, representing the number of distinct inputs the network can handle.  This is intended to match up with the parameter $n$ in the data model, where it represents the total number of level $0$ concepts, i.e., $|D_0|$.
     \item
     $\tau$, a real number, representing the firing threshold for neurons.
     \item
     $\eta$, a positive real, representing the learning rate for our learning rule.
\end{itemize}

\subsection{Network structure}

Our networks are directed graphs consisting of neurons arranged in layers, with edges directed from each layer to the next-higher layer; thus, they are feed-forward layered neural networks.

Specifically, a network $\mathcal{N}$ consists of a set $N$ of neurons, partitioned into disjoint sets $N_{\ell}, 0 \leq \ell \leq \ell'_{max}$, which we call \emph{layers}.
We refer to any particular neuron $u \in N_{\ell}$ as a \emph{layer} $\ell$ \emph{neuron}, and write $layer(u) = \ell$.
We assume (for simplicity) that each layer contains exactly $n$ neurons, that is, $|N_\ell |= n$ for every $\ell$.
We refer to the $n$ layer $0$ neurons as \emph{input neurons} and to all other neurons as \emph{non-input neurons}.
We assume total connectivity between successive layers, that is, each neuron in $N_{\ell}$, $0 \leq \ell \leq \ell'_{max} - 1$ has an outgoing edge to each neuron in  $N_{\ell+1}$, and these are the only edges.

We assume a one-to-one mapping $rep: D_0 \rightarrow N_0$,  where $rep(c)$ is the neuron corresponding to concept $c$.
That is, $rep$ is a one-to-one mapping from
 the full set of level $0$ concepts, $D_0$, to $N_0$, the set of layer $0$ neurons,
This will allow the network to receive an input corresponding to any level $0$ concept. 
See~\autoref{fig:network} for a depiction.
\begin{figure}[ht!]
  \centering
    \includegraphics[page=3,width=0.45\textwidth]{plan}
  \caption{The figure depicts the general structure of a feed-forward network. }
  \label{fig:network}
\end{figure}


We "lift" the definition of $rep$ to sets of level $0$ concepts as follows:
For any $B \subseteq D_0$, we define $rep(B) = \{rep(b) | b \in B \}$.
That is, $rep(B)$ is the set of all $reps$ of concepts in $B$.
(We will use analogous "lifting" definitions to extend other functions to sets.)

Since we know that $|C_0| = k^{\lmax+1}$, $C_0 \subseteq D_0$, and all elements of $D_0$ have $reps$ among the $n$ neurons of $N_0$, it follows that $n \geq k^{\lmax+1}$.  However, we imagine that $n$ is much larger than this, because we imagine that the total number of possible level $0$ concepts is much larger than the number that will arise in any particular execution of the network.

In \autoref{sec:probstatement}, we will consider extensions of the $rep()$ function from level $0$ concepts to higher-level concepts.  Establishing such higher-level $reps$ will be the job of a learning algorithm.

\subsection{Neuron states}

We assume that the state of each neuron consists of several \emph{state components}.  Here we distinguish between input neurons and non-input neurons.
Namely, each input neuron $u \in N_0$ has just one state component: 
\begin{itemize}
    \item \emph{firing}, with values in  $\{0,1\}$; this indicates whether or not the input neuron is currently firing.
\end{itemize}
We denote the \emph{firing} component of input neuron $u$ at integer time $t$ by $firing^u(t)$; we will sometimes abbreviate this in mathematical formulas as just $y^u(t)$.

Each non-input neuron $u \in N_{\ell}$, $1 \leq \ell \leq \ell'_{max}$, has three state components:
\begin{itemize}
\item
\emph{firing}, with values in $\{0,1\}$, indicating whether the neuron is currently firing.
\item
\emph{weight}, a real-valued column vector in $[0,1]^n$ representing current weights on incoming edges.
\item
\emph{engaged}, with values in $\{0,1\}$; indicating whether the neuron is currently prepared to learn.  As discussed in the intro, these model eligibility traces (see \cite{eligibility}).

\end{itemize}
We denote the three components of non-input neuron $u$ at time $t$ by $firing^u(t)$, $weight^u(t)$, and $engaged^u(t)$, respectively, and abbreviate these by $y^u(t)$, $w^u(t)$, and $e^u(t)$.

We also use the notation $x^u(t)$ to denote the column vector of \emph{firing} flags of $u$'s incoming neighbor neurons at time $t$.
That is,
\ifarxiv
$x^u(t)=
[ y^{v_1}(t)  y^{v_2}(t)  \dots  y^{v_n}(t) ]^{{\it T}}$,
\else
$x^u(t)=
[ y^{v_1}(t) , y^{v_2}(t) , \dots , y^{v_n}(t) ]^{{\it T}}$,
\fi
where $\{v_i\}_{i\leq n}$ are the incoming neighbors of $u$, which are exactly all the nodes in the layer below $u$. 

\subsection{Neuron transitions}

Now we describe neuron behavior, specifically, we describe how to determine the values of the state components of each neuron $u$ at time $t \geq 1$ based on values of state components at the previous time $t-1$ and on external inputs.  Again, we distinguish between input neurons and non-input neurons.

\paragraph{Input neurons:}
If $u$ is an \emph{input neuron}, then it has only one state component, the $firing$ flag.
Since $u$ is an input neuron, we assume that the value of the $firing$ flag is controlled by the network's environment and not by the network itself, that is, the value of $y^u(t)$ is set by some external input signal, which we do not model explicitly.

\paragraph{Non-input neurons:}
If $u$ is a \emph{non-input neuron}, then it has three state components, $firing$, $weight$, and $engaged$.
Whether or not neuron $u$ fires at time $t$, that is, the value of $y^u(t)$, is determined by its incoming \emph{potential} and its \emph{activation function}.

The potential at time $t$, which we denote by $pot^u(t)$ is given by the dot product of the weights and inputs at neuron $u$ at time $t-1$, that is, 
\[pot^u(t) = w^u(t-1)^T \cdot x^u(t-1) = \sum_{j=1}^n w^u_j(t-1) x^u_j(t-1).
\]
The activation function, which defines whether or not neuron $u$ fires at time $t$, is then defined by:
\[ y^u(t) =  \begin{cases}
1 & \text{if $pot^u(t) \geq \tau$}, \\
0 & \text{otherwise},
\end{cases}\]
where $\tau$ is the assumed firing threshold.

We assume that the value of the $engaged$ flag of $u$ is controlled by $u$'s environment, that is, for every $t$, the value of $e^u(t)$ is set by some input signal, which may arise from outside the network or from another part of the network.
For example, the $engaged$ flag could be used to ensure that, in any round, only one neuron is prepared to learn.\footnote{We use the term "round" to represent the activity between two consecutive times.  In particular, "round $t$" refers to the activity that takes the system from time $t-1$ to time $t$.  Thus, the potential in round $t$ means the same thing as the potential at time $t$, captured by $pot^u(t)$.}
\snote{We should make sure that we use the numbering correctly throughout, so "round t" takes us from time t-1 to time t.}
This neuron might be selected by a separate "Winner-Take-All" sub-network.

Finally, for the weights, we assume that each neuron that is engaged at time $t$ determines its weights at time $t$ according to Oja's learning rule.  
That is, if $e^u(t) = 1$, then
\begin{equation}\label{eq:Oja} 
\text{\emph{Oja's rule}:  $w^u(t) = w^u(t-1) + \eta\ z(t-1) \cdot  (  x^u(t-1) -   z(t-1)\cdot w^u(t-1) )$,} 
\end{equation}
where $\eta$ is the assumed learning rate and $z(t-1) = pot^u(t)$.\footnote{The $z(t-1)$ notation is standard for Oja's rule, so we use that in the rest of this paper when we analyze network behavior based on this rule.}
Thus, the weight vector is adjusted by an additive amount that is proportional to the learning rate and the potential, and depends on the input firing pattern, with a negative adjustment that depends on the potential and the prior weights.
%
\snote{Make sure we introduce such abbreviation conventions if/when we use them.}

\subsection{Network operation}

During execution, the network proceeds through a sequence of \emph{configurations}, $Con(0), Con(1), Con(2), \ldots$, where $Con(t)$ describes the configuration at nonnegative integer time $t$.  Each configuration specifies a state for every neuron in the network, that is, values for all the state components of every neuron.

As described above, the $y$ values for the input neurons are specified by some external source.  The $y$, $w$, and $e$ values for the non-input neurons are defined by the network specification at time $t = 0$.  For times $t > 0$, the $y$ and $w$ values are determined by the activation and learning functions described above.  The $e$ values (engagement flags) are determined by special inputs arriving from outside the network or from other sub-networks.
In our algorithms in Sections~\ref{sec:algorithms_learning} and~\ref{sec:noisylearningalgo}, they will arrive from Winner-Take-All sub-networks.

\section{Problem Statements}
\label{sec:probstatement}

In this section we define our two main problems: \emph{recognizing concept hierarchies}, and \emph{learning to recognize concept hierarchies}.  Our notion of recognition is robust to a bounded amount of noise.  The notion of learning we define in this section corresponds to noise-free learning; we extend this to noisy learning in \autoref{sec:noisy}.
In all cases, we assume that each item is represented by exactly one neuron; considering more elaborate representations is another direction for future work.

\subsection{Preliminaries}

Throughout this section, we fix constants $\lmax$, $n$, $k$, $r_1$, and $r_2$ according to the definitions for a concept hierarchy in \autoref{sec:datamodel}.  We consider a concept hierarchy $\mathcal C$, with concept set $C$ and maximum level $\lmax$, partitioned as usual into $C_0, C_1, \ldots, C_{\lmax}$. 
We also fix constants $\ell'_{max}$, $n$, $\tau$, and $\eta$ as in the definitions for a network in \autoref{sec:networkmodel}, and consider a network $\mathcal N$ as described earlier.
Thus, we allow the maximum layer number $\ell'_{max}$ for $\mathcal N$ to be different from the maximum level number $\lmax$ for $\mathcal C$, but the number $n$ of input neurons is the same as the number of level $0$ items in $\mathcal C$.

The following definition will be useful in defining our recognition and learning problems.  It expresses what it means for a particular subset $B$ of the level $0$ concepts to be "presented" as input to the network, at a certain time $t$.

\begin{definition}
[\textbf{Presented}]
\label{def:presented2}
If $B \subseteq D_0$ and $t$ is a non-negative integer, then we say that $B$ is \emph{presented at time} $t$ (in some particular execution) if, for every layer $0$ neuron $u$, the following hold:
\begin{enumerate}
\item
If $u \in rep(B)$ then $y^u(t) = 1$. 
\item
If $u \notin rep(B)$ then $y^u(t) = 0$.
\end{enumerate}
That is, all of the layer $0$ neurons in $rep(B)$ fire at time $t$, and no other layer $0$ neuron fires at time $t$.
\end{definition}

\subsection{Robust recognition}
\label{sec:prob-recog}

Here we define what it means for network $\mathcal N$ to recognize concept hierarchy $\mathcal C$.
We assume that every concept $c \in C$, at every level, has a unique representing neuron, $rep(c)$; this extends the $rep()$ function from level $0$ concepts to higher-level concepts.
For this definition, we also assume that, during the entire recognition process, the $engaged$ flags of all neurons are off, i.e., for every neuron $u$ with $layer(u) > 0$, and every $t$, $e^u(t) = 0$.

The following definition uses the two assumed values $r_1, r_2 \in [0,1]$, with $r_1 \leq r_2$.  
$r_2$ represents the fraction of children of a concept $c$ at any level that should be sufficient to support firing of $rep(c)$.  
$r_1$ is a fraction below which $rep(c)$ should not fire.

\begin{definition}
[\textbf{Robust recognition problem}]
\label{def:noisyrecognition}
Network $\mathcal N$ $(r_1,r_2)$-\emph{recognizes} a concept $c$ in concept hierarchy $\mathcal C$ provided that $\mathcal N$ contains a unique neuron $rep(c)$ such that the following holds.
Assume that $B \subseteq C_0$ is presented at time $t$.  
%

Then:
\begin{enumerate}
\item \emph{When $rep(c)$ must fire:}
If $c \in supported_{r_2}(B)$, then $rep(c)$ fires at time $t + layer(rep(c))$.
\item \emph{When $rep(c)$ must not fire:}
If $c \notin supported_{r_1}(B)$, then $rep(c)$ does not fire at time $t + layer(rep(c))$.
\end{enumerate}
We say that $\mathcal N$ $(r_1,r_2)$-\emph{recognizes} $\mathcal C$ provided that it $(r_1,r_2)$-\emph{recognizes} each concept $c$ in $\mathcal C$.
\end{definition}

The special case of $(1,1)$-recognition is interesting, since it is equivalent to the requirement that all level $0$ descendants of a concept must be present for recognition:


\begin{lemma}
\label{lem: recog-special-case}
Network $\mathcal N$ $(1,1)$-recognizes a concept $c$ in concept hierarchy $\mathcal C$ if and only if $\mathcal N$ contains a unique neuron $rep(c)$ such that the following holds.
If $B \subseteq D_0$ is presented at time $t$, then $rep(c)$ fires at time $t + layer(rep(c))$ if and only if $leaves(c) \subseteq B$.
\end{lemma}
\begin{proof}
By the definition of the robust recognition problem and \autoref{lem: support-special-case}.
\end{proof}

\subsection{Noise-free learning}
\label{sec:prob-learning}

In the learning problem, the network does not know ahead of time which particular concept hierarchy might be presented in a particular execution.  It must be capable of learning \emph{any} concept hierarchy.

In our algorithm in \autoref{sec:algorithms_learning}, in order for the network to learn a concept hierarchy $\mathcal C$, it must receive inputs corresponding to all the concepts in $C$.  Here we define how individual concepts are "shown" to the network, and then give constraints on the order in which the concepts are shown.  Such constraints are captured by the notion of a \emph{bottom-up training schedule}.
Then we state our learning guarantees, assuming a bottom-up training schedule for $\mathcal C$.

We begin by describing how an individual concept $c$ is "shown" to the network. Recall that $leaves(c)$ is defined to be $descendants(c) \cap C_0$.

\begin{definition}
[\textbf{Showing a concept}]
\label{def:shown}
Concept $c$  is \emph{shown} at time $t$ provided that the set $B = leaves(c)$ is presented at time $t$. That is, for every input neuron $u$, $y^u(t) = 1$ if and only if $u \in rep(leaves(c))$.
\end{definition}

Learning a concept hierarchy will involve showing all the concepts in the hierarchy.
Informally speaking, we assume that the concepts are shown "bottom-up".
For example, before the network is shown the concept of a head, it is shown the lower-level concepts of mouth, eye, etc.  And before it is shown the concept of a human, it is shown the lower-level concepts of head, body, legs, etc.
%
More precisely, to enable network $\mathcal N$ to learn the concept hierarchy $\mathcal C$, we assume that every concept in its concept set $C$ is shown at least $\sigma$ times, where $\sigma$ is a parameter to be specified by a learning algorithm.  
Furthermore, we assume that any concept $c \in C$ is shown only after each child of $c$ has been shown at least $\sigma$ times.
We allow the concepts to be shown in an arbitrary order and in an interleaved manner, provided that these constraints are observed.

\begin{definition}
[\textbf{$\sigma$-bottom-up training schedule}]
\label{def:bottom-up-training-schedule}
A \emph{training schedule} for $\mathcal C$ is any finite list $c_0,c_1,\ldots,c_m$ of concepts in $C$, possibly with repeats.
A training schedule is $\sigma$-\emph{bottom-up}, where $\sigma$ is a positive integer, provided that each concept in $C$ appears in the list at least $\sigma$ times, and no concept in $C$ appears before each of its children has appeared at least $\sigma$ times.
\end{definition}

Any training schedule $c_0, c_1,\ldots,c_m$ generates a corresponding sequence $B_0,B_1,\ldots,B_m$ of sets of level $0$ concepts to be presented in a learning algorithm.  Namely, $B_i$ is defined to be $rep(leaves(c_i))$.

\begin{definition}
[\textbf{$(r_1,r_2,\sigma)$-learning}]
\label{def: learningproblem}
Network $\mathcal N$ $(r_1,r_2,\sigma)$-\emph{learns} concept hierarchy $\mathcal C$ provided that the following holds.
At any time after a training phase in which all the concepts of $\mathcal C$ are shown according to a $\sigma$-bottom-up training schedule, network $\mathcal N$ $(r_1,r_2)$-recognizes $\mathcal C$.
\end{definition}

\section{Algorithms for Recognition and Noise-Free Learning}
\label{sec:algorithms}

We give algorithms for both of the problems described in \autoref{sec:probstatement}.

\subsection{Recognition}

Fix a concept hierarchy $\cal{C}$ with concept set $C$, and $r_1, r_2 \in [0,1]$, with $r_1 \leq r_2$. 
Recognition can be achieved by simply embedding the digraph induced by $\cal{C}$ in the network $\mathcal N$. See \autoref{fig:main} for an illustration.
For every $\ell$ and for every level $\ell$ concept $c$ of $\cal{C}$, we designate a unique representative $rep(c)$ in layer $\ell$ of the network.
Let $R$ be the set of all representatives, that is, $R = rep(C) = \{ rep(c)~|~c\in C\}$.
We use $rep^{-1}$ with support $R$ to denote the corresponding inverse function that gives, for every $u\in R$, the unique concept $c \in C$ with $rep(c) = u$.

If $u$ is a layer $\ell$ neuron and $v$ is a layer $\ell+1$ neuron, then we define
the edge weight $weight(u,v)$ by:
\[
weight(u,v) = 
\begin{cases}
1 & \text{ if $rep^{-1}(u) \in children(rep^{-1}(v))$},
\\
0 & \text{ otherwise.}
\end{cases}
\]
That is, we define the weights of edges corresponding to child relationships in the concept hierarchy to be $1$, and the weights of other edges to be $0$.

Finally, we set the threshold $\tau$ for every non-input neuron to be $\frac{(r_1 + r_2) k}{2}$.
It should be clear that the resulting network $\mathcal N$ solves the $(r_1,r_2)$-recognition problem:

\begin{theorem}
Network ${\cal N}$ $(r_1,r_2)$-\emph{recognizes} $\cal{C}$.
\end{theorem}

Recall that the definition of recognition, \autoref{def:noisyrecognition} says that each individual concepts $c$ in the hierarchy is recognized.  For a level $\ell$ concept $c$, the definition includes a time bound of $layer(rep(c)) = level(c) = \ell$ for recognizing concept $c$.

We note that our choice of weights in $\{0,1\}$ here is for simplicity.  Other combinations are possible, and in fact, our learning algorithm below results in different weights, approximating $\frac{1}{\sqrt{k}}$ and $0$. 

\subsection{Noise-free learning}
\label{sec:algorithms_learning}

Now we move from the simple recognition problem to the harder problem of learning.  Now we must design a network $\mathcal N$ that can learn an arbitrary concept hierarchy $\mathcal C$ with parameters as listed in \autoref{sec:datamodel} and \autoref{sec:networkmodel}, and with $\lmax \leq \ell'_{max}$.
Our algorithm utilizes \emph{Winner-Take-All (WTA)} sub-networks~\cite{lazzaro1988winner,yuille1989winner,thorpe1990spike,coultrip1992cortical,maass2000computational,wang2003k,oster2006spiking,LynchMP17a}.

\paragraph{Winner-Take-All sub-networks:}
Our algorithm uses Winner-Take-All sub-networks to select which neurons are prepared to learn at different points during the learning process. 
In this paper, we abstract from these sub-networks by simply describing their effects on the $engaged$ flags in the non-input neurons.
We give the precise requirements in~\autoref{as:WTA}.  

While the network is being trained, example concepts are "shown" to the network, one example at each time $t$, according to a $\sigma$-bottom-up training schedule as defined in \autoref{sec:prob-learning}.
We assume that, for every example concept $c$ that is shown, exactly one neuron at the appropriate layer will be engaged; this layer is the one with the same number as the level of $c$ in the concept hierarchy.
Furthermore, the neuron on that layer that is engaged is the one that has the largest potential $pot^u$. 
More precisely, in terms of timing, we assume:

\begin{assumption}[\textbf{Winner-Take-All assumption}]
\label{as:WTA}
If a level $\ell$ concept $c$ is "shown" at time $t$, then at time $t+\ell$, exactly one layer $\ell$ neuron $u$ has its \emph{engaged} state component equal to $1$, that is, it has $e^u(t+\ell) = 1$.
Moreover, $u$ is chosen so that $pot^u(t+\ell)$ is the highest potential at time $t+\ell$ among all the layer $\ell$ neurons.
\end{assumption}
\snote{This allows for the case of "ties", where more than one has the highest potential.  In that case, this says, technically, that any neuron with the highest potential could be selected.}
\snote{So our assumption is explicitly tying the layer in the network to the level of the concept.  How would a brain network guarantee this?}
\snote{By counting. If you see that $k^\ell$ concepts are shown, it's a level $\ell$ concept. Might be a little specific to our model, but in general you could also use some rule that if too much random stuff on the layer before fired, you don't update your weights.}

\snote{Frederik pointed out an algebra mistake in Section 7.  To fix that, we could include lower bounds on $r_2$ and $k$, and probably also start with slightly smaller initial weights, like $\frac{1}{k^{\lmax+1}}$.  Here is the first place where we would have to make changes, and there would be several adjustments throughout Section 7.

Besides putting in the new restrictions, the statement of this theorem would change slightly, replacing the lmax with (lmax+1).  I think.  But that doesn't change the order of magnitude.

I am making the changes now, but in green so they are easy to check or remove.}

\paragraph{Main algorithm:}
We assume that the network $\mathcal N$ starts in a clean state in which, for every neuron $u$ in layer $1$ or higher, 
$w^u(0 )= \frac{1}{k^{\corr{\lmax+1}}} \mathbf{1}$, 
where $\mathbf{1}$ is the $n$-dimensional all-one vector. 
We set the threshold $\tau$ for all neurons to be $\frac{(r_1+r_2) \sqrt{k}} {2}$, and the learning rate $\eta$ to be $\frac{1}{4k}$. 
The initial condition, threshold, learning rate, \autoref{as:WTA}, and the general model conventions for activation and learning suffice to determine how the network behaves, when shown a particular series of concepts.
Our main result is:

\snote{Several changes in the statement of the theorem.
It said that b geq 2 lmax but I don't think we used that.  I changed just to 2, which we need because of the way we have stated the lower bound on the difference $r_1 k - \lfloor r_1 k \rfloor$.
I added the assumption on r1 k.
I fixed a reversed inequality, seems like it was a typo.
I fixed up the def of sigma, which was incomplete.
Besides general cleanup, I'm ordering this consistently with the list of parameters in Section 7.2.}

\begin{theorem}[\textbf{Noise-Free Learning Theorem}]
\label{thm:noisefreelearning}
Let $\mathcal N$ be the network described above, with maximum layer $\ell'_{max}$. 
Let $b$ be an arbitrary positive real $\geq 2$.
Let $r_1, r_2$ be reals with $0 < r_1 < r_2 \leq 1$; assume that $r_1 k$ is not an integer, and $r_1 k - \lfloor r_1 k \rfloor \geq \frac{\sqrt{k}}{k^{b-1}}$.
\corr{Also assume that $r_2$ and $k$ satisfy the inequality $\frac{1}{\sqrt{k}} + \frac{1}{k} \leq \frac{r_2 \sqrt{k}}{2}$.}
\footnote{\corr{This last assumption can be satisfied by a variety of different combinations of assumptions on $r_2$ and $k$ individually, such as $r_2 \geq \frac{1}{2}$ and $k \geq 6$, or $r_2 \geq \frac{1}{4}$ and $k \geq 11$.}\snote{Double-check my arithmetic.}}
Let $\epsilon = \frac{r_2-r_1}{r_1+r_2}$.

Let $\mathcal C$ be any concept hierarchy, with maximum level $\lmax \leq \ell'_{max}$.
%
%
Let $\sigma = 
\frac{4}{3 \eta k}(\corr{(\lmax+1)} \log(k)) 
+ \frac{3}{\eta k\varepsilon} 
+ \frac{b \log(k)}{ \log(\frac{16}{15})}$.
Thus, $\sigma$ is
$O\left(\frac{1}{\eta k} \left(\lmax \log(k) + \frac{1}{\varepsilon}\right) + b \log(k)\right)$. 

Then $\mathcal N$ $(r_1,r_2,\sigma)$-learns concept hierarchy $\mathcal C$.

\end{theorem}
That is, unwinding the definition of $(r_1,r_2,\sigma)$-learning,
at any time after a training phase in which all the concepts of $\mathcal C$ are shown according to a $\sigma$-bottom-up training schedule, network $\mathcal N$ $(r_1,r_2)$-recognizes $\mathcal C$.
    
A rigorous analysis can be found in \autoref{sec:noisefreeanalysis}; the main idea of the analysis is as follows.
We first prove some direct consequences of Oja's rule (\autoref{lem:wincr}, \autoref{lem:invariants}, and \autoref{lem:structure}).  These quantify the weight changes for a single neuron involved in learning a single concept, assuming that all of its child concepts have already been learned.  
In particular, we show that the weights change quickly so that they approximate either $1/\sqrt{k}$ or $0$, depending on whether or not the weights correspond to neurons that represent child concepts.
\snote{I don't think we can say that the weights actually converge to 1/sqrt(k), at least not from the statement of Lemma 7.3.  That just claims that we wind up in a range that is near that bound, but the discrepancy depends on epsilon, which is fixed (in terms of r1 and r2) and not approaching 0.  Am I right?}
\snote{Lemma 7.3 shows this comes very close to 1/sqrt k for any fixed eps. So I think the statement above is quite all right.}

We next build on these lemmas to describe, in \autoref{lem: main-noise-free}, the learning (i.e., weight changes) that occur throughout the network in the course of the entire execution. 
What makes this challenging is that we allow "incomparable" concepts to be shown in an interleaved manner; the only constraint is that, for every concept $c$, child concepts of a concept $c$ must be shown sufficiently many times before $c$ is shown.
%
In order to prove that all concepts are learned correctly despite these challenges, we use an involved yet elegant five-part induction. 
Finally, in \autoref{sec: main-invariant} we put everything together and show that the network successfully $(r_1,r_2,\sigma)$-learns the concept hierarchy.

\section{Extension to Noisy Learning}
\label{sec:noisy}

We extend our model, algorithm, and analysis to noisy learning. The idea is that we should be able to learn a concept even if we do not see all the child concepts at every time. For example, we could expect to learn the concept of a "human" even if we sometimes see only the "legs" and  "body", and other times see only the "head" and "legs" etc.

To model this, we assume that, in order to show a concept $c$, we show a random $p$-fraction of its sub-concepts. 
Formally, we use the following recursive marking procedure to determine which inputs should be presented to the network:
We begin by marking $c$.
Then, proceeding recursively, for any marked concept, we mark a random $p$-fraction of the sub-concepts. 
The recursion terminates when a subset of the leaves of $c$ are marked.  The inputs presented to the network are the representations of the marked leaves of $c$.

\subsection{Modifications to the model}
\label{sec:noisylearningmodel}

Formally, our model is as follows. 
Recall that in \autoref{def:shown}, we assumed that when a concept $c$ is shown, that \emph{all} $reps$ of the leaves of $c$ fire. 
We now weaken this assumption, as follows.

\begin{definition}
[\textbf{$p$-noisy-showing a concept}]
\label{def:noisy-shown}
Concept $c$  is $p$-\emph{noisy-shown} at time $t$, where $p \in (0,1]$, provided that a subset $B \subseteq leaves(c)$ produced by the random function $mark(c,p)$ is presented at time $t$.  \\
Random function $mark(c,p)$ is defined recursively based on the level of $c$:  If $level(c) = 0$, then $mark(c,p) = \{ c \}$.  If $level(c) \geq 1$, then choose a subset $C'$ consisting of exactly $\lceil p k \rceil$ children of $c$, uniformly at random, 
and let $mark(c,p) = \bigcup_{c' \in C'} mark(c',p)$.
\end{definition}

In the noisy case, we need an upper bound ($\sigma_2$ in the following definition) on the number of times a concept is noisy-shown. See the discussion in the footnote before \autoref{thm:noisylearning} for more details.
\snote{?  This footnote is located before the statement of Theorem 6.4.  You are talking here about footnote 6, not 7.}
\snote{fixed. I meant that the text appears after it, but I guess your understanding is more natural.}

\begin{definition}
[\textbf{$(\sigma_1,\sigma_2)$-bottom-up training schedule}]
A training schedule is $(\sigma_1,\sigma_2)$-\emph{bottom-up}, where $\sigma_1$ and $\sigma_2$ are positive integers, $\sigma_1 \leq \sigma_2$, provided that each concept in $C$ appears in the list at least $\sigma_1$ times and no more than $\sigma_2$ times, and no concept in $C$ appears before each of its children has appeared at least $\sigma_1$ times.
\end{definition}

\begin{definition}
[\textbf{$(r_1,r_2,\sigma_1,\sigma_2,p)$-noisy learning}]
Network $\mathcal N$ $(r_1,r_2,\sigma_1,\sigma_2,p)$-\emph{noisy-learns} concept hierarchy $\mathcal C$ provided that the following holds.
At any time after a training phase in which all the concepts of $\mathcal C$ are $p$-noisy-shown according to a $(\sigma_1,\sigma_2)$-bottom-up training schedule, network $\mathcal N$ $(r_1,r_2)$-recognizes $\mathcal C$.
\end{definition}



\subsection{Noisy Learning Algorithm}\label{sec:noisylearningalgo}

The algorithm is exactly the same as in \autoref{sec:algorithms_learning}, except that here we use $p$-noisy showing (\autoref{def:noisy-shown})  instead of ordinary showing (\autoref{def:shown}).
We prove that our modified algorithm is robust in that it works even for our notions of noisy showing and noisy learning.

Our theorem for noisy learning, \autoref{thm:noisylearning}, differs from \autoref{thm:noisefreelearning} in that we guarantee "correctness" only in cases where each concept is noisy-shown at most $n^6$ times, that is, in cases where the network $(r_1,r_2,\sigma,n^6,p)$-noisy learns the concept hierarchy.
\footnote{Note that we assume that every concept is shown at most $n^6$ times.  This is natural since if we consider a number $T$ of rounds that is of order exponential in $n$, then at some point $t\leq T$ it is very likely that the weights will be unfavorable for recognition.
This can happen since in such a large time frame, it's very likely that there will be a long sequence of runs in which the same representatives are simply (due to bad luck) not shown. The network will forget about their importance.  %
This is also partly the reason why the learning rate in the following theorem is smaller than the one of the noise-free counterpart: the smaller learning rate guarantees that during the first $n^6$ rounds no unlikely sequence occurs that is very `bad'.
}
Let $\bar{w}=1/\sqrt{pk+1-p}$.
Our algorithm uses the learning rate $\eta = \frac{(\frac{\delta p \bar{w}}{20})^3}{ 64 T k^2 p^3 }$ and 
the firing threshold $\tau = r_2 k (\bar{w}-2\delta)$, where $\delta = \bar{w}(r_2-r_1)/50$.

We now state our main theorem in the noisy-learning setting. 

\begin{theorem}[\textbf{Noisy-Learning Theorem}]
\label{thm:noisylearning}
Let $\mathcal N$ be the network described  in \autoref{sec:networkmodel}, with maximum layer $\ell'_{max}$.
Let $r_1, r_2$ be reals with $0 < r_1 < r_2 \leq 1$; assume that $r_2-r_1 \geq 1/k$ and $k\geq 2$.
%
Let $\mathcal C$ be any concept hierarchy, with maximum level $\lmax \leq \ell'_{max}$ and a total of $|C|$ concepts.
Let $\sigma = c' \frac{k^6}{p^6\delta^3}\left(\lmax \log(k) + 
  \log(|C|n/\delta) \right)$, for some large enough constant $c'$.

Then, w.h.p., $\mathcal N$ $(r_1,r_2,\sigma, n^6,p)$-noisy-learns concept hierarchy $\mathcal C$.\footnote{We define w.h.p in this paper to be $1 - \frac{1}{n}$.}
\end{theorem}

\subsection{Proof idea}\label{sec:noisyproofidearough}

In the  presence of noise, many of the properties of the noise-free case no longer hold, rendering the proof significantly more involved.  Here we give a rough outline of our proof; details appear in \autoref{sec:noisyanalysis}.

In the analysis we only consider the learning of one concept, as the interleaving of different concepts is no different than in the noise-free case and hence we do not repeat that analysis. Therefore, in the reminder we fix one concept.

First, we bound the worst-case change of potential during a period of $T$ rounds (where the concept is shown), provided it is initially within certain bounds.  We later show that it will stay throughout the first $n^6$ rounds where the concept is shown.
\snote{This description is assuming an overall bound of $n^6$ rather than $n^6$ for each concept.  But I think we can make the stronger claim that it works provided we use the bound of $n^6$ for each concept.  Seems more natural to me, since it corresponds to the definition, and I suspect it is what you actually use.}
\snote{Made clear that this per concept}

We aim to derive bounds on the change of the weight of a single edge during such a period. 
It turns out that the way the weights change depends highly on the other weights, which makes the analysis non-trivial.
For this reason, we refrain from showing convergence of each weight separately. 
\snote{What?  The previous sentence doesn't make sense to me.}
\snote{Better?}
\snote{Fine.}
Instead we use the  following potential function $\psi$.
to show that the max and min weight convergence towards $\bar{w}=\frac{1}{\sqrt{ p k + 1-p   } }$ and $0$ respectively.
Fix an arbitrary time $t$ and let $w_{min}(t)$ and $w_{max}(t)$ be the minimum and maximum weights among $w_{1}(t),w_{k}(t), \dots, w_{k}(t)$, respectively. 
%
Let $\psi(t) = \max \left\{  \frac{w_{max}(t)}{\bar{w}}, \frac{\bar{w}}{w_{min}(t)}  \right\}. $

Note that, in contrast to the noise-free case, weights belonging to representatives of sub-concepts converge to 
$ \bar{w} $ instead to $1/\sqrt{k}$.

Our goal is to show that the above potential decreases quickly until it is very close to $1$.
Showing that the potential decreases is involved, since one cannot simply use a worst-case approach, due to the terms in Oja's rule being non-linear and potentially having a high variance, depending on the distribution of weights.  
Instead, the key to showing that $\psi$ decreases is to carefully use the randomness over the input vector and to carefully bound the non-linear terms.
Bounding these non-linear terms tightly presents a major challenge.
To overcome it, we show that the changes of the weights form a Doob martingale allowing us to use Azuma-Hoeffding inequality to get asymptotically almost tight bounds on the change of the weights during the $T$ rounds.
The proof can be found in  \autoref{sec:noisyanalysis}.


\section{A Lower Bound}\label{sec:lowerbounds} 

\snote{I think we probably could extend to a neuron model with strengths instead of just firing indicators.  Basically, by replacing the sums of weights by dot products.}

Our results so far demonstrate how concept hierarchies with $\lmax$ levels can be represented robustly by networks with the same number of layers, and how such representations can be learned, even in the presence of noise.
We would also like lower bound theorems saying that $\lmax$ layers are necessary for robust representation, under suitable restrictions. 

In this section, we give a first step toward such a result, \autoref{thm:imposs-2-vs-1}.
It says that a network $\mathcal N$ with maximum layer $1$ cannot recognize a concept hierarchy $\mathcal C$ with maximum level $2$.
This bound depends only on the requirement that $\mathcal N$ should recognize $\mathcal C$ according to our definition for noisy recognition in \sref{Definition}{def:noisyrecognition}. 
That definition says that the network must tolerate bounded noise, as expressed by the ratio parameters $r_1$ and $r_2$.  Our result assumes reasonable constraints on the values of $r_1$ and $r_2$.
Note that the bound does not involve learning, only recognition.

{A preliminary generalization of this result to more levels and layers appears in~\cite{arxiv-v3}.
However, in addition to the basic definition of noisy recognition, this generalization uses a strong technical assumption about disjointness of certain sets of triggered neurons.
This assumption might be reasonable, in that it is guaranteed by our learning algorithms in \autoref{sec:algorithms_learning}; however, we think it is too strong and would prefer to weaken it to, say, a simple limitation on the number of neurons at each layer in the network.  We leave this task for future work.}

\subsection{Assumptions for the lower bound}
\label{sec: lower-bound-assumptions}

Here we list explicitly the assumptions that we use for our lower bound result, \autoref{thm:imposs-2-vs-1}.
We state these assumptions in a general way, in terms of a particular concept hierarchy $\mathcal C$ with concept set $C$ and any number $\lmax$ of levels, and an arbitrary network $\mathcal N$ with any number $\ell'_{max}$ of layers.
However, our lower bound result, \autoref{thm:imposs-2-vs-1}, refers to just the special case of two levels and one layer.
These assumptions capture the idea that concept hierarchy $\mathcal C$ is $(r_1,r_2)$-recognized by network $\mathcal N$.

\begin{enumerate}
    \item  Every concept $c \in C$ has a unique designated neuron $rep(c)$ in the network.  (In general, it might be in any layer, regardless of the level of $c$.)
    \item  Let $B$ be any subset of $C_0$.  If $c \in supported_{r_2}(B)$, then presentation of $B$ at time $t$ results in firing of $rep(c)$ at time $t+layer(rep(c))$.
    \item  Let $B$ be any subset of $C_0$.  If $c \notin supported_{r_1}(B)$, then presentation of $B$ at time $t$ does not result in firing of $rep(c)$ at time $t+layer(rep(c))$.
\end{enumerate}

\snote{We assume throughout the paper that firing proceeds layer-by-layer, each layer at one time.  So when we talk about a neuron $u$ at level $\ell$ being triggered to fire, we mean that this occurs at time $t+layer(u)$, where $t$ is the time when the triggering set of level $0$ concepts is presented. We might need to clarify this timing issue in other definitions and claims. Here, I added a footnote.}


Throughout this section, we assume the model presented in \autoref{sec:datamodel} and~\autoref{sec:networkmodel}.
Furthermore, since we are considering recognition only, and not learning, we assume that the $engaged$ state components are always equal to $0$.
\snote{We are using a simplified version of our general model, simpler than what we allow for our algorithmic results.  Namely, neuron states are just binary `firing status" values, representing whether the neuron is firing or not at a particular time.  And we have no extra facilities such as WTA modules, just neurons arranged in layers.  This corresponds to the basic SNN model of Musco et al, ITCS paper 2017, new version in arXiv 2019, my composition paper 2018.}
Also throughout this section, we assume that $r_1$ and $r_2$ satisfy the following constraints:
\begin{enumerate}
    \item $0 \leq r_1 \leq r_2 \leq 1$.
    \item $r_1 k$ is not an integer; define $r_1'$ so that $r_1' k = \lfloor r_1 k \rfloor$.
    \item Define $r_2'$ so that $r_2' k = \lceil r_2 k \rceil$.
    \snote{I think here it's OK to allow the case where $r_2' = r_2$.  Check the details.}
    \item $(r_2')^2 \leq 2 r_1' - (r_1')^2$.
\end{enumerate}
we think these constraints are reasonable.  For example, for $k = 10$, $r_1 = .51$ and $r_2 = .8$ satisfy these conditions.  Or $r_1 = \frac{1}{3}$ and $r_2 = \frac{2}{3}$.

\subsection{Impossibility for recognition for two levels and one layer}
\label{sec:imposs-2-levels-1-layer}

We consider an arbitrary concept hierarchy $\mathcal C$ with maximum level $2$ and concept set $C$.
We assume a  (static) network $\mathcal N$ with maximum layer $1$, and total connectivity from layer $0$ neurons to layer $1$ neurons.
For such a network and concept hierarchy, we get a contradiction to the noisy recognition problem  in \autoref{sec:prob-recog}, for any values of $r_1$ and $r_2$ that satisfy the constraints given in \autoref{sec: lower-bound-assumptions}.
For the problem requirements, we use only Assumptions 1-3 from \autoref{sec: lower-bound-assumptions}.

\begin{theorem}
\label{thm:imposs-2-vs-1}
Assume that $\mathcal C$ has maximum level $2$ and $\mathcal N$ has maximum layer $1$.
Assume that $r_1, r_2, r_1', r_2'$ satisfy the constraints in \autoref{sec: lower-bound-assumptions}.
Then $\mathcal N$ does not recognize $\mathcal C$, according to Assumptions 1-3.
\end{theorem}
 
\begin{proof}
Assume for contradiction that $\mathcal{N}$ recognizes $\mathcal{C}$. 
Let $c$ denote any one of the concepts in $C_2$, i.e., a level $2$ concept in $C$.
Then $c$ has $k$ children, each of which has $k$ children of its own, for a total of $k^2$ grandchildren.

Each of the $k^2$ grandchildren must have a $rep$ in layer $0$, but neither $c$ nor any of its $k$ children do, because layer $0$ is reserved for level $0$ concepts.  
So in particular, $rep(c)$ is a layer $1$ neuron.
By the structure of the network, this means that the only inputs to $rep(c)$ are from layer $0$ neurons.
Since we assume total connectivity, we have an edge from each layer $0$ neuron to $rep(c)$.
We define:
\begin{itemize}
    \item  $W(b)$, for each child $b$ of $c$ in the concept hierarchy:  The total weight of all edges $(u,rep(c))$, where $u$ is a layer $0$ neuron that is the $rep$ of a child of $b$.
    \item  $W$:  The total weight of all the edges $(u,rep(c))$, where $u$ is a layer $0$ neuron that is a $rep$ of a grandchild of $c$.  In other words, $W = \Sigma_{b \in children(c)} W(b)$.
\end{itemize}

We consider two scenarios.
In Scenario A (the "must-fire scenario"), we choose input set $B$ to consist of enough leaves of $c$ to force $rep(c)$ to fire, that is, we ensure that $c \in supported_{r_2}(B)$, while trying to minimize the total weight incoming to $rep(c)$.  
Specifically, we choose the $r_2' k \geq r_2 k$ children $b$ of $c$ with the smallest values of $W(b)$.
And for each such $b$, we choose its $r_2' k$ children with the smallest
weights. 
Let $B$ be the union of all of these $r_2' k$ sets of $r_2' k$ grandchildren of $c$.
Since $r_2' k \geq r_2 k$, it follows that $c \in supported_{r_2}(B)$.

\noindent\emph{Claim 1:}
In Scenario A, the total incoming potential to $rep(c)$ is at most $(r_2')^2 W$.

In Scenario B (the "can't-fire scenario"), we choose input set $B$ to consist of leaves of $c$ that force $rep(c)$ not to fire, that is, we ensure that $c \notin supported_{r_1}(B)$, while trying to maximize the total weight incoming to $rep(c)$.
Specifically, we choose the $r_1' k < r_1 k$ children $b$ of $c$ with the largest values of $W(b)$, and we include all of their children in $B$.  
For each of the remaining $(1 - r_1') k$ children of $c$, we choose its $r_1' k < r_1 k$ children with the largest weights and include them all in $B$.
Since $r_1' k$ is strictly less than $r_1 k$, it follows that $c \notin supported_{r_1}(B)$.

\noindent
\emph{Claim 2:}  In Scenario B, the total incoming potential to $rep(c)$ is at least $(r_1') W + (1 - r_1') r_1' W = (2r_1' - (r_1')^2) W$. 

\noindent
\emph{Proof of Claim 2:}
We define:
\begin{itemize}
\item
$W_1$:  The total of the weights $W(b)$ for the $r_1' k$ children $b$ of $c$ with the largest values of $W(b)$.
\item
$W_2 = W - W_1$:  The total of the weights $W(b)$ for the remaining $(1 - r_1') k$ children of $c$.
\item
$W_3$:  We know that $W_1 \geq r_1' W$, since $W_1$ gives the total weight for the $r_1' k$ children of $c$ with the largest weights, out of $k$ children.  
Define $W_3 = W_1 - r_1' W$; then $W_3$ must be nonnegative.
\end{itemize}

Then the total incoming potential to $rep(c)$ is 
\begin{align*}
  &\geq W_1 + r_1' W_2,\\
&= r_1' W + W_3 + r_1'(W - W_1),\\
&= r_1' W + W_3 + r_1'(W - W_3 - r_1' W),\\
&= 2 r_1' W - (r_1')^2 W + (1 - r_1')W_3, \\
&\geq 2 r_1' W - (r_1')^2 W, \\
&= (2 r_1' - (r_1')^2) W, \\
\end{align*}
as needed.

\noindent
\emph{End of proof of Claim 2}

Now, Claim 1 implies that the threshold $\tau$ of neuron $rep(c)$ must be at most $(r_2')^2 W$,
since it must be small enough to permit the given $B$ to trigger firing of $rep(c)$.
On the other hand, Claim 2 implies that the threshold must be strictly greater than
$(2 r_1' - (r_1')^2) W$, since it must be large enough to prevent the given $B$ from triggering firing of $rep(c)$.  
So we must have
\[(2 r_1' - (r_1')^2) W < \tau \leq (r_2')^2 W,\]
which implies that
\[2 r_1' - (r_1')^2 < (r_2')^2.\]
But this contradicts our assumption that $(r_2')^2 \leq 2 r_1'  - (r_1')^2$.
\end{proof}

\section{Conclusions and Future Work}
\label{sec:futurework}


In this paper, we have proposed a theoretical model for recognizing and learning hierarchically-structured concepts in synchronous, feed-forward layered Spiking Neural Networks.
Our networks use Oja's learning rule for adjusting synapse weights.
Based on this model, we have presented two unsupervised learning algorithms, one for noise-free learning and one that allows bounded noise. 
Both algorithms learn concepts in a bottom-up manner, but allow arbitrary interleaving in learning of incomparable concepts.
We have analyzed both algorithms in detail.

The representations produced by these algorithms are certain types of embeddings of the hierarchical concept structure in the neural network.
These representations support robust concept recognition, even when some of the inputs are missing. 
We have also provided a preliminary lower bound on the number of layers, saying that two-level concepts cannot be recognized robustly in one-level networks.


This paper represents a first step towards a theory of representation and learning for hierarchically-structured concepts in SNNs.
Our representations and algorithms appear to be generally consistent with experimental results in computer vision and neuroscience. 
However, our model is highly abstract and makes several simplifying assumptions, in the interests of exposing the key ideas and simplifying the analysis:  for instance, we assume that concepts are strictly tree-structured, that every concept has the same number of children, and that the learning rule is applied without error.
To make the results more realistic, one should, of course, loosen these assumptions.


The results in this paper suggest numerous directions for future research:

\paragraph{Extensions to our results:}
One can consider more flexible orders in which concepts in a hierarchy can be learned, based on a larger class of training schedules.
Is it possible to learn higher-level concepts before learning low-level concepts? 
How does the order of learning affect the time required to learn?
Another interesting issue is robustness of the networks, for example, to presentation of a few "extraneous" inputs that are not part of the concept being shown, to noise in calculating potentials, or to failures of neurons or synapses.

Also, our algorithms use some auxiliary capabilities, such as Winner-Take-All, in order to select neurons for learning; it would be interesting to combine our algorithms with network implementations of these auxiliary capabilities in order to obtain complete, self-contained networks that solve the learning problem "from scratch".
Finally, we would like to strengthen the lower bound results to apply to many levels and layers.

\paragraph{Variations in the network model:}
Our networks have a simple layered structure; it would be interesting to consider some natural variations.
For example, instead of all-to-all connections between consecutive layers, what happens to the results if one assumes a smaller number of randomly-determined connections between layers?
Also, in our networks, all edges go from one layer $\ell$ to the next higher layer $\ell+1$.  How do the results change if one allows edges to go from layer $\ell$ to any higher layer? 

What would be the impact on the results of allowing feedback edges from each layer $\ell$ to the next-lower layer $\ell-1$?  How would the costs of recognizing and learning concepts change based on feedback from representations of higher-level concepts?
Finally, what would be the effect of using other variants of Hebbian learning rules besides Oja's rule?

\paragraph{Variations in the data model:}
Another interesting research direction is to consider variations on the structure of concept hierarchies.
How do the results change if we allow different numbers of children for different nodes, or allow a level $\ell$ concept to have children at any level smaller than $\ell$, rather than just level $\ell-1$?
What happens if a concept hierarchy need not be a tree, but may include a bounded amount of overlap between the sets of children of different concepts?

It would be interesting to understand more generally what kinds of logical structures can be learned by synchronous SNNs.
In our concept hierarchies, each level $\ell+1$ concept corresponds to the "and" of several level $\ell$ concepts.
What if we allow concepts that correspond to "ors", or "nors", of other concepts?
Similar questions were suggested by Valiant~\cite{valiant2000circuits}, in terms of a different model.
%
%
Also, in addition to learning individual concepts, it would be interesting to consider learning relationships between concepts, such as association, causality, or sequential order.
\snote{Linguistic structure?}

\paragraph{Different forms of representation:}
In this paper, each concept $c$ is represented by just one neuron $rep(c)$.  
An interesting extension, which may be more biologically plausible, would be to allow the representation of each concept $c$ to be a more elaborate "code" consisting of a particular set of neurons that fire.
What are the theoretical advantages and costs of such codes, compared to simpler single-neuron representations?

\bibliography{biblio}

\appendix

\section{Analysis of Noise-free Learning}\label{sec:noisefreeanalysis}

Here we present our analysis for the noise-free learning algorithm in 
\autoref{sec:algorithms}.
In \autoref{sec: weight-change-individual}, we describe how incoming weights change for a particular neuron when it is presented with a consistent input vector. 
In \autoref{sec: main-invariant-noise-free}, we prove our main invariant, saying how neurons get bound to concepts, when neuron firing occurs, and how weights change, during the time when the network is learning.
In \autoref{sec: main-invariant}, we use that invariant to prove \autoref{thm:noisefreelearning}.

\subsection{Weight Change for Individual Neurons}
\label{sec: weight-change-individual}

In this subsection we give a series of three lemmas that describe how incoming weights change for a particular neuron when it is presented with a consistent input vector during execution of our noise-free learning network.  
Throughout this subsection, we consider a single neuron $u$ with $layer(u) \geq 1$.

We begin by considering how weights change in a single round.
\autoref{lem:wincr} describes how the weights change for firing neighbors, and for non-firing neighbors.
In this lemma, we consider a neuron $u$ with weight vector $w(t-1)$ and input vector $x(t-1)$, both at time $t-1 \geq 0$.
Write $z(t-1)$ for the dot product of $w(t-1)$ and $x(t-1)$, which represents the incoming potential in round $t$.
We assume that the $engaged$ component, $e(t)$, is equal to $1$. 
We give bounds on the new weights for $u$ at time $t$, given by $w(t)$.

\begin{lemma}
\label{lem:wincr}
Let $F \subseteq \{1,\ldots,n\}$, with $|F| = k$.
Assume that:
\begin{enumerate}
    \item $x_i(t-1) = 1$ for every $i \in F$ and $x_i(t-1) = 0$ for every $i \notin F$.  That is, exactly the incoming neighbors in $F$ fire at time $t-1$.
    \item All weights $w_i(t-1), i \in F$ are equal, and all weights $w_i(t-1), i \notin F$ are equal.
    \item  For every $i \in F$, $0 < w_i(t-1) < \frac{1}{\sqrt{k}}$.
    \item  For every $i \notin F$, $w_i(t-1) > 0$.
    \item  $0 < \eta \leq \frac{1}{4k}$.
\end{enumerate}
Then:
\begin{enumerate}
    \item  All weights $w_i(t), i \in F$ are equal, and all weights $w_i(t), i \notin F$ are equal.
    \item For every $i \in F$, $w_i(t) > w_i(t-1)$.
    \item For every $i \in F$, $w_i(t) < \frac{1}{\sqrt{k}}$.
    \item For every $i \notin F$, $w_i(t) < w_i(t-1)$.
    \item For every $i \notin F$, $w_i(t) > 0$.
\end{enumerate}
\end{lemma}

\begin{proof}
Note that $z(t-1) < k \frac{1}{\sqrt{k}} = \sqrt{k}$, because of the assumed upper bound for each $w_j(t-1)$ and the fact that $|F| = k$.
Similarly, we have that $z(t-1) > 0$.

Part 1 is immediate by symmetry---all components for $i \in F$ are changed by the same rule, based on the same information.

For Part 2, consider any $i \in F$.
Since $z(t-1) < \sqrt{k}$ and $w_i(t-1) < \frac{1}{\sqrt{k}}$, the product $z(t-1)\  w_i(t-1) < 1$.
Then by Oja's rule:
 \[
w_i(t) = w_i(t-1) + \eta  z(t-1) (1 -  z(t-1) w_i(t-1) )  
       > w_i(t-1)  + \eta z(t-1) \cdot 0
       = w_i(t-1), 
\]
as needed.

For Part 3, again consider any $i \in F$.
Since $w_i(t-1) < \frac{1}{\sqrt{k}}$, we may write 
$w_i(t-1) = \frac{1}{\sqrt{k}} - \lambda$ for some $\lambda > 0$.
Then by symmetry, for every $j \in F$, we have $w_i(t-1) = \frac{1}{\sqrt{k}} - \lambda$. 
We thus have that 
\begin{align*}
w_i(t) &=  w_i(t-1) + \eta  z(t-1) (1 -  z(t-1) w_i(t-1) )  \\
&= w_i(t-1) + \eta k \cdot \left(\frac{1}{\sqrt{k}}-\lambda \right) \left(1-  k\left( \frac{1}{\sqrt{k}}-\lambda \right)^2\right) \\
&= w_i(t-1) + \eta k \cdot \left(\frac{1}{\sqrt{k}}-\lambda\right) \left(1-  k\left( \frac{1}{k}-\frac{2\lambda}{\sqrt{k}} + \lambda^2\right)\right)\\
&< w_i(t-1) + \eta k \cdot \left(\frac{1}{\sqrt{k}}\right)  2\lambda\sqrt{k} \\
&\leq w_i(t-1) + \frac{\lambda}{2} \\
&< 1/ \sqrt{k},
\end{align*}
as needed.

For Part 4, consider any $i \notin F$.
We have
\begin{align*}
w_i(t) &=  w_i(t-1) + \eta  z(t-1) (0 -  z(t-1) w_i(t-1) )  \\
&=  w_i(t-1)(1 - \eta z(t-1)^2 ) \\
&< w_i(t-1),
\end{align*}
as needed.
\snote{There was a repeated line above, which I removed.}

Finally, for Part 5, again consider any $i \notin F$.
We then have: 
\begin{align*}
w_i(t) &=  w_i(t-1) + \eta  z(t-1) (0 -  z(t-1) w_i(t-1) ) \\
&= w_i(t-1) (1 - \eta z(t-1)^2 ) \\
&> w_i(t-1) (1 - \eta k ), \mbox{ since } z(t-1) < \sqrt{k} \\
&\geq w_i(t-1) (1 - \frac{k}{4k} ), \mbox{ since } \eta \leq \frac{1}{4k} \\
&= \frac{3}{4} w_i(t-1) \\
&> 0,
\end{align*}
as needed.
\end{proof}

\autoref{lem:invariants} extends \autoref{lem:wincr} to any number of steps.
This lemma assumes that the same $x$ inputs are given to the given neuron $u$ at every time. 
When we apply this later, in the proof of \autoref{lem: main-noise-free}, it will be in a context where these inputs may occur at separated times, namely, the particular times at which $u$ is actually engaged in learning.  At the intervening times, $u$ will not be engaged in learning and therefore will not change its weights.

\begin{lemma}
\label{lem:invariants}
Let $F \subseteq \{1,\ldots,n\}$, with $|F| = k$.
Assume that:
\begin{enumerate}
    \item For every $t \geq 0$, $x_i(t) = 1$ for every $i \in F$ and $x_i(t) = 0$ for every $i \notin F$. 
    \item All weights $w_i(0)$ are equal.
    \item $0 < w_i(0) < \frac{1}{\sqrt{k}}$ for every $i$.
    \item $0 < \eta \leq \frac{1}{4k}$.
\end{enumerate}
Then for any $t \geq 1$:
\begin{enumerate}
    \item All weights $w_i(t), i \in F$ are equal, and all weights $w_i(t), i \notin F$ are equal.
    \item $0 < w_i(t) < \frac{1}{\sqrt{k}}$ for every $i$.
    \item For every $i \in F$, $w_i(t) > w_i(0)$.
    \item For every $i \notin F$, $w_i(t) < w_i(0)$.
\end{enumerate}
\end{lemma}

\autoref{lem:structure} gives quantitative bounds on the amount of weight increase and weight decrease over many rounds, again for a single neuron $u$ involved in learning a single concept. 
We use notation $w(t), x(t), z(t)$ as before.
We assume that $x(t)$ is the same at all times $t = 0,1,\ldots$, and assume that the engaged component $e(t)$ is equal to $1$ at all times $t$.

\begin{lemma}[Learning Properties]
\label{lem:structure}
Let $F \subseteq \{1,\ldots,n\}$ with $|F| = k$.  
Let $\varepsilon \in (0,1]$.
Let $b$ be a positive integer.
Let $\sigma = \frac{4}{3 \eta k}(\corr{(\lmax+1)} \log(k)) 
+ \frac{3}{\eta k\varepsilon} 
+ \frac{b \log(k)}{ \log(\frac{16}{15})}$. 
Thus, $\sigma$ is
$O\left(\frac{1}{\eta k} \left(\lmax \log(k) + \frac{1}{\varepsilon}\right) + b \log(k)\right)$.
Assume that:
\begin{enumerate}
    \item For every $t \geq 0$, $x_i(t) = 1$ for every $i \in F$, $x_i(t) = 0$ for every $i \notin F$, and $e(t) = 1$. 
    \item All weights $w_i(0)$ are equal to $\frac{1}{k^{\lmax}}$.
    \item $\eta = \frac{1}{4k}$.\footnote{This is a very precise assumption but it could be weakened, at a corresponding cost in running time.}
\end{enumerate}
Then for every $t \geq \sigma$, the following hold:
\begin{enumerate}
\item 
For any $i \in F$, we have
$w_i(t) \in [ \frac{1}{(1+\varepsilon)\sqrt{k}},\frac{1}{\sqrt{k}}]$.
\item 
For any $i \notin F$,  we have $w_i(t) \in [0,\frac{1}{k^{\lmax+b}}]$. 
\end{enumerate} 
\end{lemma}

\begin{proof}
We first show Part 1.
\autoref {lem:invariants} implies the upper bound of $\frac{1}{\sqrt{k}}$, so it remains to show the lower bound.  We do this is two steps, first increasing the weight to an intermediate target value $\frac{1}{2 \sqrt{k}}$ and then to the real target value $\frac{1}{(1+\varepsilon)\sqrt{k}}$.  These two steps use different arguments.

For the first step, we begin with Claim 1, which bounds the number of rounds required to double the weight $w_i$, for $i \in F$, when $w_i$ is not "too close" to the target weight $\frac{1}{\sqrt{k}}$.

\noindent
\emph{Claim 1:}  Assume that $i \in F$.  
For any positive integer $j$, the number of rounds needed to increase $w_i$ from $\frac{1}{2^{j+1} \sqrt{k}}$ to $\frac{1}{2^{j} \sqrt{k}}$ is at most $\frac{4}{3 \eta k}$.

\noindent\emph{Proof of Claim 1:}
Since all the weights are the same and $\frac{1}{2^{j+1} \sqrt{k}} \leq w_i(t-1) \leq \frac{1}{2 \sqrt{k}}$, we get:
\begin{align*}
 w_i(t) &= w_i(t-1) + \eta  z(t-1) \cdot(1 - z(t-1)\cdot w_i(t-1) ) \\
 &= w_i(t-1)  + \eta  k w_i(t-1) (1- k w_i^2(t-1)) \\
 &\geq w_i(t-1)  + \frac{\eta k}{2^{j+1}\sqrt{k} }  (1- k \frac{1}{4k}) \\ 
 &= w_i(t-1)  + \frac{\eta k}{2^{j+1}\sqrt{k} }  (3/4).
\end{align*}

Increasing $w_i$ from $\frac{1}{2^{j+1} \sqrt{k}}$ to $\frac{1}{2^{j} \sqrt{k}}$ means we must increase it by an additive amount of $\frac{1}{2^{j+1} \sqrt{k}}$.
We have just shown that each round increases $w_i$ by at least 
$\eta  k \frac{1}{2^{j+1}\sqrt{k} }  (3/4)$.
Thus, the number of rounds required to double $w_i$ from $\frac{1}{2^{j+1} \sqrt{k}}$ to $\frac{1}{2^{j} \sqrt{k}}$ is at most
$\frac{1}{2^{j+1} \sqrt{k}}$ divided by
$\eta  k \frac{1}{2^{j+1}\sqrt{k} }  (3/4)$, which is
$\frac{4}{3 \eta k}$. \\
\emph{End of proof of Claim 1.}

Now we can prove the first step, bounding the number of rounds required for the weight to reach at least $\frac{1}{2 \sqrt{k}}$:

\vspace{.2cm}
\noindent
\emph{Claim 2:}
For $i \in F$, the number of rounds required to increase $w_i$ from the starting value $\frac{1}{k^{\lmax}}$ to the intermediate target value $\frac{1}{2 \sqrt{k}}$ is at most $\frac{4}{3 \eta k}(\corr{(\lmax+1)} \log(k))$.

\noindent
\emph{Proof of Claim 2:}
By applying Claim 1 $\corr{(\lmax+1)} \log(k)$ times. \\
\emph{End of Proof of Claim 2.}

Next, for the second step, we bound the number of rounds required to increase $w_i$, $i \in F$, from 
$\frac{1}{2\sqrt{k}}$ to $\frac{1}{(1+\varepsilon)\sqrt{k}}$. This time, of course, depends on $\varepsilon$.  

\noindent
\emph{Claim 3:}
For $i \in F$, the number of rounds required to increase $w_i$ from the intermediate target value $\frac{1}{2\sqrt{k}}$ to the final target value $\frac{1}{(1+\varepsilon)\sqrt{k}}$ is at most $\frac{3}{\eta k\varepsilon}$.

\noindent
\emph{Proof of Claim 3:}
The argument is generally similar to that for Claim 1, but now using the fact that $\frac{1}{2\sqrt{k}} \leq w_i(t-1) \leq \frac{1}{(1+\varepsilon)\sqrt{k}}$:

\begin{align*}
w_i(t) &= w_i(t-1) + \eta  z(t-1) (1 - z(t-1) w_i(t-1)) \\
&= w_i(t-1) + \eta  k w_i(t-1) (1- k w_i^2(t-1)) \\
&\geq w_i(t-1) + \frac{\eta k}{2\sqrt{k} } \left( 1-\frac{1}{(1+\varepsilon)^2} \right)  \\
&= w_i(t-1) + \frac{\eta \sqrt{k}}{2} \left( 1-\frac{1}{(1+\varepsilon)^2} \right)  \\
&\geq w_i(t-1) + \frac{\eta \sqrt{k}}{2} \frac{\varepsilon}{3}, \\
&= w_i(t-1) + \frac{\eta \sqrt{k}\varepsilon}{6},
\end{align*}
where we used the fact that $(1-1/(1+x)^2) \geq x/3$ for $0 \leq x\leq 1$.
It follows that the total time to increase $w_i$ from its initial value $\frac{1}{2 \sqrt{k}}$ to the target value $\frac{1}{(1+\varepsilon) \sqrt{k}}$ is at most

\begin{align*}
\left(\frac{1}{(1+\varepsilon) \sqrt{k}} - \frac{1}{2 \sqrt{k}} \right) \cdot
\frac{6}{\eta \sqrt{k}\varepsilon} 
&= \frac{1 -\varepsilon}{2(1+\varepsilon) \sqrt{k}} \cdot
\frac{6}{\eta \sqrt{k}\varepsilon} 
= \frac{6 (1 -\varepsilon)}{2 (1 +\varepsilon) \eta k\varepsilon} 
\leq \frac{3}{\eta k\varepsilon}.
\end{align*}

\noindent
\emph{End of Proof of Claim 3.}

It follows that the total number of rounds for Part 1 is at most the sum of the bounds from Claims 2 and 3, or
\[
\frac{4}{3 \eta k} \left( \corr{(\lmax+1)} \log(k) \right) + \frac{3}{\eta k\varepsilon}, 
\]
which is
$O\left( \frac{1}{\eta k} (\lmax \log(k) + \frac{1}{\varepsilon}) \right)$.

Note that once the weights for indices in $F$ reach their target values, they never decrease below those values.  This follows from strict monotonicity shown in
\autoref{lem:invariants}.

We now turn to proving Part 2.
\autoref {lem:invariants} implies the lower bound, so it remains to show the upper bound.

We consider what happens after the increasing weights (for indices in $F$) have already reached the level $\frac{1}{2 \sqrt{k}}$, and then bound the number of rounds for the decreasing weights to decrease to the desired target $\frac{1}{k^{\lmax + b}}$.  
The reason we choose the level $\frac{1}{2 \sqrt{k}}$ for the increasing weights is that this is enough to guarantee that $z$ is "large enough" to produce a sufficient amount of decrease.
For this part, we use our assumed lower bound on $\eta$.  

\noindent
\emph{Claim 4:}
For $i \notin F$, the number of rounds required to decrease $w_i$ from the starting weight 
\corr{$\frac{1}{k^{\lmax+1}}$} to $\frac{1}{k^{\lmax+b}}$ is at most 
$\frac{b \log_2 k}{\log_2{\frac{16}{15}}}$, which is $O(b \log(k))$.

\noindent
\emph{Proof of Claim 4:}
Considering a single round, we get:
\begin{align*}
w_i(t) &= w_i(t-1)(1 - \eta  z(t-1)^2) \\
&\leq w_i(t-1)\left(1 - \frac{1}{4k} \left(\frac{\sqrt{k}}{2}\right)^2\right) \\
&= w_i(t-1)\left(1 - \frac{1}{16}\right) 
= w_i(t-1)\frac{15}{16}.
\end{align*}
The inequality uses the facts that $\eta \geq \frac{1}{4k}$ and $z(t-1) \geq k (\frac{1}{2\sqrt{k}}) = \frac{\sqrt{k}}{2}$.

Thus, the weight decreases by a factor of $15/16$ at each round.
Now consider the number of rounds needed to reduce from $\frac{1}{k^{\corr{\lmax+1}}}$ to the target weight $\frac{1}{k^{\lmax+b}}$.  
This number is bounded by 
$\frac{b \log_2 k}{\log_2{\frac{16}{15}}}$,
which is $O(b \log(k))$, as claimed. \\
\noindent\emph{End of Proof of Claim 4.}

Summing the bounds for Part 1 (increasing) and Part 2 (decreasing), we see that the total number of rounds to complete all the needed increases and decreases is at most 
\[
\frac{4}{3 \eta k} \left( \corr{(\lmax+1)} \log(k) \right) + \frac{3}{\eta k\varepsilon} + \frac{b \log_2 k}{\log_2{\frac{16}{15}}},
\]
which is
$O\left( \frac{1}{\eta k} (\lmax \log(k) + \frac{1}{\varepsilon}) + b \log(k)\right)$, as needed.

\snote{Note that in this final bound, we are being somewhat pessimistic, since we are ignoring the fact that the increases and decreases happen concurrently, and just summing the bounds.}
\end{proof}

\subsection{Main Invariants}
\label{sec: main-invariant-noise-free}

In this section, we give a key lemma, \autoref{lem: main-noise-free}, which describes key properties of the algorithm with respect to engagement, weight settings, and firing.  This lemma deals with the network as a whole, and draws upon the lemmas in \autoref{sec: weight-change-individual} for properties involving learning by individual neurons.
\autoref{lem: main-noise-free} relies on assumptions about the input, captured by our $\sigma$-bottom-up training definition, and also about the settings of $engagement$ flags.  

For the rest of \autoref{sec:noisefreeanalysis}, we use the following assumptions about the various parameter settings:
\begin{enumerate}
    \item  The concept hierarchy consists of $\lmax$ levels.
    \item  The network consists of $\ell'_{max}$ levels, with $\lmax \leq \ell'_{max}$.
    \item  $b$ is a positive real $\geq 2$.
    \item  $r_1$, $r_2$ satisfy $0 < r_1 < r_2 \leq 1$, and $r_1 k$ is not an integer; more strongly, we assume the technical condition that $r_1 k - \lfloor r_1 k \rfloor \geq \frac{\sqrt{k}}{k^{b-1}}$.  \corr{Furthermore, we assume that $\frac{1}{\sqrt{k}} + \frac{1}{k} \leq \frac{r_2 \sqrt{k}}{2}$.}
    \snote{Definitely an awkward-sounding assumption, though I don't think it's very unreasonable.  It's needed for showing non-firing, in the proof of the main theorem.}
    \item  $\varepsilon = \frac{r_2 - r_1}{r_1 + r_2}$.
    \item  $\tau = \frac{(r_1+r_2) \sqrt{k}} {2}$.
    \item  $\eta = \frac{1}{4k}$.
    \item  $\sigma$, for the $\sigma$-bottom-up training schedule definition, is equal to  $\frac{4}{3 \eta k}(\corr{(\lmax+1)} \log(k)) 
+ \frac{3}{\eta k\varepsilon} 
+ \frac{b \log(k)}{ \log(\frac{16}{15})}$.
Thus, $\sigma$ is
$O\left(\frac{1}{\eta k} \left(\lmax \log(k) + \frac{1}{\varepsilon}\right) + b \log(k)\right)$. 
\end{enumerate}

We use the following assumption about the settings of the engagement flags.
\begin{assumption}\label{ass:flag}
For every time $t$ and layer $\ell$, a neuron $u$ on layer $\ell \geq 1$ is engaged (i.e., $u.engaged = 1$) at time $t$, if and only if both of the following hold:
\begin{enumerate}
    \item A level $\ell$ concept was shown at time $t-\ell$.
    \item Neuron $u$ is selected by the WTA at time $t$.
\end{enumerate}
\end{assumption}
Recall that, by \autoref{as:WTA}, the WTA selects exactly one layer $\ell$ neuron at time $t$.  This, together with \autoref{as:WTA}, implies that exactly one layer $\ell$ neuron will be engaged at time $t$.

We also define the point at which a particular layer $\ell$ neuron $u$ gets "bound" to a particular level $\ell$ concept $c$.
Namely, we say that a layer $\ell$ neuron $u$, $\ell \geq 1$, "binds" to a level $\ell$ concept $c$ at time $t$ if $c$ is presented for the first time at time $t - \ell$, and $u$ is the neuron that is engaged at time $t$.
At that point, we define $rep(c) = u$.

Here is a simple auxiliary lemma, about unbound neurons.

\begin{lemma}
\label{lem: unbound-neurons}
Let $u$ be a neuron with $layer(u) \geq 1$.
Then for every $t \geq 0$, the following hold:
\begin{enumerate}
\item 
If $u$ is unbound at time $t$, then all of $u$'s incoming weights at time $t$ are the initial weight $\frac{1}{k^{\corr{\lmax+1}}}$.
\item
If $u$ is unbound at time $t$, then $u$ does not fire at time $t$.
\end{enumerate}
\end{lemma}

We are now ready to prove our main lemma.  It has five parts, whose proofs are intertwined.

\begin{lemma}
\label{lem: main-noise-free}
Consider any particular execution of the network in which inputs follow a $\sigma$-bottom-up training schedule.
For any $t \geq 0$, the following properties hold.
\begin{enumerate}
\item
The $rep()$ mapping from the set $C$ of concepts to the set $N$ of neurons a is one-to-one mapping; that is, for any two distinct concepts $c$ and $c'$ for which $rep(c)$ and $rep(c')$ are both defined by time $t$, we have $rep(c) \neq rep(c')$.
\item
For every concept $c$ with $level(c) \geq 1$, every showing of $c$ at a time $\leq t - level(c)$, leads to the same neuron $u = rep(c)$ becoming engaged at time $t$.
\item
For every concept $c$ with $level(c) \geq 1$, and any $t' \geq 1$, if $c$ is shown at time $t-level(c)$ for the $t'$-th time, then the following are true at time $t$:
   \begin{enumerate}
   \item 
   Neuron $u=rep(c)$ has weights in 
   $\left( \frac{1}{k^{\corr{\lmax+1}}}, \frac{1}{\sqrt{k}} \right)$ for all neurons in $rep(children(c))$, and weights in $\left(0, \frac{1}{k^{\corr{\lmax+1}}}\right)$ for all other neurons.
   \item 
   If $t'\geq \sigma$, then $u$ with $u=rep(c)$ has weights in $\left[\frac{1}{(1+\varepsilon)\sqrt{k}},\frac{1}{\sqrt{k}}\right]$ for all neurons in $rep(children(c))$,  
   and weights in $\left[0, \frac{1}{k^{\lmax+b}}\right]$ for all other neurons.
   \end{enumerate}
\item For every concept $c$, if a proper ancestor of $c$ is shown at time $t-level(c)$, then $rep(c)$ is defined by time $t$, and fires at time $t$.
\item 
For any neuron $u$,  the following holds.
If $u$ fires at time $t$, then there exists $c$ such that $u = rep(c)$ at time $t$, and an ancestor of $c$ is shown at time $t-layer(u)$.
(This ancestor could be $c$ or a proper ancestor of $c$.)
\end{enumerate}
\end{lemma}

\begin{proof}
First observe that, by \autoref{ass:flag}, every representative $rep(c)$ is on the layer equal to $level(c)$.
We prove the five-part statement of the lemma by induction on $t$.

\noindent
\emph{Base:}  $t=0$.

For Part 1, the only concepts for which $reps$ are defined at time $0$ are level $0$ concepts, and these all have distinct $reps$ by assumption.
For Parts 2 and 3, note that $level(c) \geq 1$ implies that the times in question are negative, which is impossible; so these are trivially true.
For Part 4, it must be that $level(c) = 0$ (to avoid negative times), and a proper ancestor of $c$ is shown at time $0$.  Then the layer $0$ neuron $rep(c)$ fires at time $0$, by the definition of "showing".

For Part 5, first note that at time $0$ no neurons at layers $\geq 1$ are bound, so by \autoref{lem: unbound-neurons}, they cannot fire at time $0$.  
Since we assume that $u$ fires at time $0$, it must be that $layer(u) = 0$,  which implies that $u = rep(c)$ for some level $0$ concept $c$.    
Then, since $u$ fires at time $0$, by definition of "showing", an ancestor of $c$ must be shown at time $0$. 

\noindent\emph{Inductive step:} Assume the five-part claim holds for time $t-1$ and consider time $t$. We prove the five parts one by one.

For Part 1, let $c$ and $c'$ be any two distinct concepts for which $rep(c)$ and $rep(c')$ are both defined by time $t$.  We must show that $rep(c) \neq rep(c')$.
If both $rep(c)$ and $rep(c')$ are defined by time $t-1$, then by the inductive hypothesis, Part 1, $rep(c) \neq rep(c')$ at time $t-1$.  Since the $reps$ do not change, this is still true at time $t$, as needed.
So the only remaining possibility for conflict is that one of these two concepts, say $c'$, already has its $rep$ defined by time $t-1$ and the other concept, $c$, does not, and $rep(c)$ becomes defined at time $t$, to be the same neuron as $rep(c')$.
But we claim that, because of the weight settings, $rep(c)$ must be defined at time $t$ to be a neuron that is unbound at time $t-1$.

So suppose that $u$ is the neuron that gets defined to be $rep(c)$ at time $t$; we argue that $u$ must be unbound at time $t-1$.
Write $\ell =level(c)$; then also $layer(u) = \ell$.
By \autoref{ass:flag}, the $engaged$ flag gets set at time $t$ for $u$, and for no other layer $\ell$ neurons.
Since $c$ is shown at time $t-\ell$, by the $\sigma$-bottom-up assumption, each child of $c$ must have been shown at least $\sigma$ times prior to time $t - \ell$.
Then by the inductive hypothesis, Parts 4 and 5, the layer $\ell-1$ neurons "fire correctly" at time $t-1$, that is, all neurons in the set $rep(children(c))$ fire and no other layer $\ell - 1$ neuron fires, at time $t-1$.
This firing pattern implies that every layer $\ell$ neuron that is already bound strictly prior to time $t$ has incoming potential in round $t$ that is strictly less than $k$ times the initial weight,
by the inductive hypothesis Part 3(a) and by the disjointness of the concepts.
On the other hand, every layer $\ell$ neuron that is unbound at time $t-1$ has incoming potential equal to $k$ times the initial weight, by \autoref{lem: unbound-neurons}.
By assumption, there must be at least one unbound neuron available.
It follows that the neuron $u$ that is chosen by the WTA is unbound at time $t-1$, and so cannot be the same as the already-bound neuron $rep(c')$.  \\

For Part 2, let $c$ be any concept with $level(c) \geq 1$, and write $\ell = level(c)$.  We must prove that any showing of $c$ at any time $\leq t - \ell$ leads to the same neuron $u = rep(c)$ becoming engaged.
If $c$ is not shown at time precisely $t - \ell$, then the claim follows directly from the inductive hypothesis, Part 2.
So assume that $c$ is shown at time $t - \ell$.
If $t - \ell$ is the first time that $c$ is shown, then $rep(c)$ first gets defined at time $t$, so the conclusion is trivially true (since there is only one showing to consider).

It remains to consider the case where $rep(c)$ is already defined by time $t-1$.
Then, by the inductive hypothesis, Part 2, we know that any showing of $c$ at a time $\leq t-1 - \ell$ leads to neuron $rep(c)$ becoming engaged.
We now argue that the same $rep(c)$ is also selected at time $t$.
As in the proof of Part 1, the $engaged$ flag is set at time $t$ for exactly one layer $\ell$ neuron; we claim that this chosen neuron is in fact the previously-defined $rep(c)$.
As in the proof for Part 1, we claim that all neurons in the set $rep(children(c))$ fire and no other layer $\ell - 1$ neuron fires at time $t-1$.
Then $rep(c)$ has incoming potential in round $t$ that is strictly greater than $k$ times the initial weight, 
by the inductive hypothesis, Part 3(a).
On other hand, every other layer $\ell$ neuron has incoming potential that is at most $k$ times the initial weight, again by the inductive hypothesis, Part 3(a).
It follows that $rep(c)$ has a strictly higher incoming potential in round $t$ than any other layer $\ell$ neuron, and so is the chosen neuron at time $t$. \\

For Part 3, let $c$ be any concept with $level(c) \geq 1$, and write $\ell = level(c)$.
Let $t' \geq 1$.
Assume that $c$ is shown at time $t-\ell$ for the $t'$-th time.
We must show:
\begin{enumerate}[(a)]
\item 
Neuron $u = rep(c)$ has weights 
in $\left( \frac{1}{k^{\corr{\lmax+1}}}, \frac{1}{\sqrt{k}}\right)$ for all neurons in $rep(children(c))$,  
and weights in $\left(0, \frac{1}{k^{\corr{\lmax+1}}}\right)$ for all other neurons.
\item 
If $t'\geq \sigma$, then $u$ with $u=rep(c)$ has weights in $\left[\frac{1}{(1+\varepsilon)\sqrt{k}},\frac{1}{\sqrt{k}}\right]$ for all neurons in $rep(children(c))$,  
and weights in $\left[0, \frac{1}{k^{\lmax+b}}\right]$ for all other neurons.
\end{enumerate}

For both parts, we use Part 2 (for $t$, not $t-1$)
to infer that every showing of $c$ at a time $\leq t - level(c)$ leads to the same neuron $u = rep(c)$ being engaged.
Thus, neuron $u$ has been engaged $t'$ times as a result of showing $c$, up to time $t$.

For Part (a), fix any $t' \geq 1$. 
Then we may apply \autoref{lem:invariants}, with $F = rep(children(c))$, to conclude that the incoming weights for $u$ are in the claimed intervals.  Here we use the fact that the initial settings $w_i(0)$ are equal to $\frac{1}{k^{\corr{\lmax+1}}}$
For Part (b), assume that $t' \geq \sigma$.  
Then we may apply \autoref{lem:structure}, with $F = rep(children(c))$, to conclude that the incoming weights for $u$ are in the claimed intervals. \\

For Part 4, let $c$ be any concept, and assume that $c^*$, a proper ancestor of $c$, is shown at time $t-level(c)$.
We must show that $rep(c)$ is defined by time $t$, and that it fires at time $t$.

Since $c^*$ is shown at time $t - level(c)$, by the definition of a $\sigma$-bottom-up schedule, that means $c$ was shown at least $\sigma$ times by time $t-level(c)-1$. 
This implies that $rep(c)$ is defined by time $t-1$, and so, by time $t$.
Moreover, since $c$ was shown at least $\sigma$ times by time $t-level(c)-1$,
by the inductive hypothesis, Part 3(b), at time $t-1$, $rep(c)$ has incoming weights at least 
$\frac{1}{(1+\varepsilon)\sqrt{k}}$ for all neurons in $rep(children(c))$.
By the inductive hypothesis, Part 4, the neurons in $rep(children(c))$ fire at time $t-1$ since $c^*$ is also a proper ancestor of all children of $c$.
Therefore, in round $t$, the potential of $rep(c)$ is at least $k \cdot \frac{1}{(1+\varepsilon)\sqrt{k}}$, which by our assumptions on the values of the parameters means that the potential is at least $\tau$, which implies that $u$ fires at time $t$. \\

For Part 5, fix an arbitrary neuron $u$ and suppose that $u$ fires at time $t$.
We must show that there is some concept $c$ such that $u = rep(c)$ at time $t$, and a (not necessarily proper) ancestor of $c$ is shown at time $t - layer(c)$.
Since $u$ fires at time $t$, by \autoref{lem: unbound-neurons}, we know that $u$ is bound at time $t$; let $c$ be the (unique) concept such that $u = rep(c)$.
The firing of $u$ at time $t$ is due to the showing of some concept, say $c^*$, at time $t-layer(u)$.

Let $R$ be the subset of $rep(children(c))$ that fire at time $t-1$.
We claim that $|R| \geq 2$; that is, at least two $reps$ of children of $c$ must fire at time $t-1$.
For, if at most one $rep(c')$ for a child of $c$ fires at time $t-1$, then by the inductive hypothesis, Part 3(a), the total potential incoming to $u$ in round $t$ would be at most 
\[
\frac{1}{\sqrt{k}} + \frac{k^{\lmax}}{k^{\corr{\lmax+1}}} = \frac{1}{\sqrt{k}} + \corr{\frac{1}{k} \leq \frac{r_2 \sqrt{k}}{2}} \leq \tau,
\] 
where $\tau$ is the threshold for firing. 

\snote{F:  Are we sure? It seems right, but there might be some setting of $r_2$ and $r_1$ where this is true if $k=1$ or so? Hopefully we assume $k>1$.}

\snote{N:  You are right.  How to fix this?  I'd rather not change the threshold, since that setting seems clean.  

If we kept the starting weights as they were, all we can do is play with relationships involving r1, r2 and k.  
For instance, we can assume $r2 \geq 1/2$ and $k \geq 25$.  Yes, it seems that we need k to be that large in order to make the given inequality work out right.  Such a large bound is unpleasant for us, though perhaps biologically reasonable.

So, let's consider changing the starting weights.
If we make the starting weights a little smaller, say $1/k^{lmax + 1}$, and keep the assumption that $r2 \geq 1/2$, that would reduce the size of the needed $k$, I think to $k \geq 6$.

If we make the starting weights $1/k^{lmax + 2}$ then I think we can reduce a bit more, say to $k \geq 4$.

I eventually decided to make the changes for the starting weight of $1/k^{lmax + 1}$, changes all marked in green for you to check or if we have to undo.  Instead of specific bounds on r2 and k, I just assume the inequality that they need to satisfy.  We can add more about the specific choices of r2 and k that would guarantee this inequality.
}

\snote{F: The threshold can be lower bounded by $2r_2k/2 = r_2 k. $ If we assume $r_2 \geq 1/\sqrt{k}$, then we have a potential of $1/\sqrt{k}+ 1  \leq \sqrt{k} \leq r_2 k /2 \leq \tau $ for $k\geq 3$. If you also want $k=2$, then we could assume $r_2 \geq 2/(\sqrt{k})$ and everything works!}

\snote{N:  You say that the threshold is lower bounded by $r_2 k$ but that doesn't seem right.  Threshold is $(r_1 + r_2) \sqrt{k}/2$.  Assuming $r_1$ has no lower bound, that means you are saying that $r_2 k \leq r2 \sqrt{k} / 2$.  Not right.

I suggest that we assume $r_2 \geq 1/2$ and $k \geq 6$, as well as changing the denominator in the initial weights to $k^{\lmax+1}$ instead of $k^{\lmax}$.
Then we get the needed inequality, which now becomes: 
\[
\frac{1}{\sqrt{k}} + \frac{1}{k} \leq \frac{\sqrt{k}}{4} \leq \frac{r_2 \sqrt{k}}{2} \leq \tau.
\] 

What else would need to change?  Here is what I see: \\
1.  Section 5.2, paragraph "Main algorithm", change the initial weight to 
$\frac{}{k^{\lmax + 1}}$.\\
2.  Section 5.2, Theorem 5.3, add restrictions that $r_2 \geq 1/2$ and $k \geq 6$.\\
3.  Change the concrete upper bound in that theorem to include the term (lmax+1)log(k), instead of just lmax log(k).  This doesn't change the O() bound.\\
4.  I don't think Lemmas 7.1 or 7.2 need to change.  Right? \\
5.  Lemma 7.3 statement should have a change to (lmax+1)log(k) in the concrete upper bound, no change in the O() bound. \\
6.  Claim 1 seems to be unchanged.  For Claim 2, change to "applying Claim 1 (lmax+1) log k times".  Claim 3 is unchanged.  \\
7.  Claim 4 statement:  Change the mention of the starting weight to the new starting weight.
The upper bound numerator could be changed to (b-1) log k, but we might as well just leave it, put a suppressed note there in case we ever care.\\
8.  Final concrete bound in the proof, change to (lmax+1) log(k).\\
9. Section 7.2, in the list of assumptions, add the lower bound on $r_2$ and the lower bound on $k$.  Change the bound sigma to include the $(lmax+1) log(k)$ term.\\
10.  Lemma 7.5, part 1, change mention of initial weight.\\
11.  Lemma 7.6, part 3(a), change the two mentions of $k^{\lmax}$ to $k^{lmax+1}$.\\
12.  Lemma 7.6 proof, no changes I can see in base case or Parts 1 and 2.  For Part 3, statement of 1 needs two changes to denominators for new starting weights \\
13.  In the proof of 3(a), change the mention of the starting weight.\\
14.  Part 5 proof, change the equation as I wrote just above. Add mention that we are using the bounds on r1 and k here. \\
15.  I don't see any changes needed in the final proof of Theorem 5.3.  Because this deals with the situation where learning is finished, which uses the final bounds and doesn't mention the starting bounds.

That sound like a lot of changes but they are easy and I can make them quickly, if you are willing to check them.

Based on all the chatter, I suggest that we start by lowering the initial weights to $k^{lmax+1}$  That will require many of the changes I list above.  Then in the list of assumptions for Section 7.2, I can just add the assumption:
$1/\sqrt{k} + 1/k \leq r_2 \sqrt{k} / 2$.  Then in a footnote, we can describe various combinations of lower bound assumptions on $r_2$ and $k$ that suffice to guarantee the needed inequality.  Sound OK?  I didn't hear back from you, am anxious to finish this, so I went ahead and made the changes, in green.
}

Therefore, $|R| \geq 2$; let $u'$ and $u''$ be any two distinct elements of $R$.
Since $u'$ and $u''$ fire at time $t-1$, by \autoref{lem: unbound-neurons}, we know that both are bound at time $t-1$; let $c'$ and $c''$ be the respective concepts such that $u' = rep(c')$ and $u'' = rep(c'')$.
We know that $c' \neq c''$ because each concept gets only one $rep$ neuron, by the way that $rep$ is defined.
Note that the firing of both $u'$ and $u''$ must be due to the showing of the same concept $c^*$ at time $(t-1) - (layer(u)-1) = t - layer(u)$.
Then by the inductive hypothesis, Part 5, applied to both $u'$ and $u''$, we see that $c^*$ must be an ancestor of both $c'$ and $c''$.
Therefore, $c^*$ must be an ancestor of the common parent $c$ of $c'$ and $c''$,
as needed.
 
This completes the overall proof of the lemma.  
\end{proof}

\subsection{Proof of \autoref{thm:noisefreelearning}}
\label{sec: main-invariant}

Now we use \autoref{lem: main-noise-free} to prove our main theorem about noise-free learning, \autoref{thm:noisefreelearning}.

\vspace{-.2cm}
\begin{proof}
By assumption, all the concepts in the hierarchy are shown according to a $\sigma$-bottom-up training schedule.
This implies, by \autoref{ass:flag}, that after the schedule, all the concepts in the hierarchy have $reps$ in the corresponding layers, that is, for each $c \in C$, $layer(rep(c)) = level(c)$.
Also, by \autoref{lem: main-noise-free}, Part 3(b), the weights after the schedule are set as as follows: 
For every concept $c$ with $level(c) \geq 1$, all incoming weights of $rep(c)$ from the $reps$ of its children, i.e., the neurons in $rep(children(c)$, are in the range $[\frac{1}{(1+\eps)\sqrt{k}}, \frac{1}{\sqrt{k}}]$, 
and weights from all other neurons (on layer $level(c)-1$) are in the range $[0, \frac{1}{k^{\lmax+b}}]$. 
\snote{You had changed b to lmax in some places.  I changed it back to b where I noticed it.  We should discuss the reasons for having particular values of b.  Until we understand this better, I'd rather keep the flexibility of having the parameter b.}

We must argue that the resulting network $\mathcal N$ $(r_1,r_2)$-recognizes the concept hierarchy $\mathcal C$, according to \autoref{def:noisyrecognition}.
This has two directions, saying that certain neurons must fire and certain neurons must not fire, at certain times, when a particular subset $B \subseteq C_0$ is presented.
So suppose that a particular subset $B \subseteq C_0$ is presented at time $t$.

\noindent
\emph{Neurons that must fire:}
We must show that the $rep$ of any concept $c$ in $supported_{r_2}(B)$ fires at time $t+level(c)$ (see \autoref{def:support} for the definition of $supported$).
We prove this by induction on the level number $\ell$, $1 \leq \ell \leq \lmax$, showing that the $rep$ of each level $\ell$ concept in $supported_{r_2}(B)$ fires at time $t+level(c)$.

For the base case, consider a level $1$ concept $c \in supported_{r_2}(B)$; then $rep(c)$ is in $layer$ $1$.
Since $c \in supported_{r_2}(B)$, it means that $|children(c) \cap B| \geq r_2 k$, that is, at least $r_2 k$ children of $c$ are in $B$.
As noted above,
the $rep$ of each of these children is connected to $rep(c)$ by an edge with weight at least 
$ \frac{1}{(1+\varepsilon) \sqrt{k} }$,
which yields a total incoming potential for $rep(c)$ in round $1$ of at least
\[ 
\frac{r_2 k}{(1+\varepsilon) \sqrt{k} } 
= \frac{r_2 \sqrt{k}}{1+\varepsilon}.
\]

To show that $rep(c)$ fires at time $t+1$, it suffices to show that the right-hand side is at least as large as the firing threshold $\tau = \frac{(r_1+r_2) \sqrt{k}}{2}$.
That is, we must show that
$
\frac{r_2}{1+\varepsilon} \geq \frac{r_1+r_2}{2}.
$
Plugging in the expression for $\varepsilon$, we get that:
\[
\frac{r_2}{1+\varepsilon}
= \frac{r_2}{1 + \frac{r_2-r_1}{r_1+r_2}}
= \frac{r_1 + r_2}{2},
\]
as needed. 

For the inductive step, consider $\ell \geq 2$ and assume by induction that the $rep$ of any level $\ell-1$ concept in $supported_{r_2}(B)$ fires at time $t+\ell-1$.
Consider a level $\ell$ concept $c \in supported_{r_2}(B)$.
Since $c \in supported_{r_2}(B)$, it means that $|children(c) \cap B_{\ell-1}| \geq r_2 k$, using notation from \autoref{def:support}, that is, at least $r_2 k$ children of $c$ are in $supported_{r_2}(B)$.
By the inductive hypothesis, the $reps$ of all of these children of $c$ fire at time $t+\ell-1$.
As noted above, the $rep$ of each of these children is connected to $rep(c)$ by an edge with weight at least 
$ \frac{1}{(1+\varepsilon) \sqrt{k} }$,
which yields a total incoming potential for $rep(c)$ in round $t+\ell$ of at least
\[ 
\frac{r_2 k}{(1+\varepsilon) \sqrt{k} } 
= \frac{r_2 \sqrt{k}}{1+\varepsilon}.
\]
Arguing as in the base case, this is at least as large as the firing threshold $\tau$, as needed to guarantee that $rep(c)$ fires at time $t+\ell$.

\snote{We have to be careful about the numbering of times, and the use of the term "rounds".  I found a few issues.  Recall that we have pot(t) meaning the same thing as z(t-1), that is, the potential (at time t) is calculated from a dot product of quantities defined at time $t-1$.  At some point, we switched, presumably for clarity, to using "potential in round t" for this quantity, meaning the "round" that goes from time t to time t+1.  I went back and added a footnote to explain this use of "round" the first time we do this.  And I hope that we have been careful enough throughout to use this notation consistently.  I will keep trying to check...}

\noindent
\emph{Neurons that must not fire:}  
We must show that the $rep$ of any concept $c$ that is not in $supported_{r_1}(B)$ does not fire at time $t+level(c)$.  
Again we prove this by induction on the level number $\ell$, $1 \leq \ell \leq \lmax$, showing that the $rep$ of each level $\ell$ concept that is not in $supported_{r_1}(B)$ does not fire at time $t + level(c)$.

For the base case,  consider a level $1$ concept $c \notin supported_{r_1}(B)$; then $rep(c)$ is in layer $1$.
Since $c \notin supported_{r_1}(B)$, it means that $|children(c) \cap B| < r_1 k$, which implies that $|children(c) \cap B| \leq \lfloor r_1 k \rfloor$.
As noted above,
the $rep$ of each of these children is connected to $rep(c)$ by an edge with weight at most $\frac{1}{\sqrt{k}}$.
Also, there are at most $k^{\lmax+1}$ other level $0$ firing neurons, since $B \subseteq C_0$, and all the weights on edges connecting these to $rep(c)$ are
at most $\frac{1}{k^{\lmax+b}}$.
Therefore, the total incoming potential for $rep(c)$ in round $t+1$ is at most
\[\frac{{\lfloor r_1 k \rfloor}}{\sqrt{k}} + \frac{k^{\lmax+1}}{k^{\lmax+b}} 
= \frac{{\lfloor r_1 k \rfloor}}{\sqrt{k}} + \frac{1}{k^{b-1}}.
\]

Now we use the technical assumption that $r_1 k - \lfloor r_1 k \rfloor \geq \frac{\sqrt{k}}{k^{b-1}}$.
Then the right hand side of the last inequality is at most
\[
\frac{r_1 k - \frac{\sqrt{k}}{k^{b-1}}}{\sqrt{k}} + \frac{1}{k^{b-1}} = r_1 \sqrt{k} < \frac{(r_1+r_2) \sqrt{k}}{2} = \tau,\]
which implies that $rep(c)$ does not fire.

\snote{Write $a$ as an abbreviation for $r_1 k - \lfloor r_1 k \rfloor$.
Then the right hand side of the last inequality is at most
\[r_1 \sqrt{k} - \frac{a}{\sqrt{k}} + \frac{1}{k^{b-1}}
\leq r_1 \sqrt{k} < \frac{(r_1+r_2) \sqrt{k}}{2} = \tau,\]
which implies that $rep(c)$ does not fire. }

\snote{
Hmm I'd have written
$\frac{{\lfloor r_1 k \rfloor}}{\sqrt{k}} + \frac{1}{k^{b-1}} \leq 
\frac{r_1 k - \frac{\sqrt{k}}{k^{b-1}}}{\sqrt{k}} + \frac{1}{k^{b-1}} = r_1 \sqrt{k} < \frac{(r_1+r_2) \sqrt{k}}{2} = \tau$
}

For the inductive step, consider $\ell \geq 2$ and assume by induction that 
the $rep$ of any level $\ell-1$ concept that is not in $supported_{r_1}(B)$ does not fire at time $t + \ell - 1$.
Consider a level $\ell$ concept $c \notin supported_{r_1}(B)$.
Since $c \notin supported_{r_1}(B)$, it means that
$|children(c) \cap B_{\ell-1}| < r_1 k$, that is, the number of children of $c$ that are in $supported_{r_1}(B)$ is less than $r_1 k$.
As noted above, the $rep$ of each of these children is connected to $rep(c)$ by an edge with weight at most $\frac{1}{\sqrt{k}}$.

Now consider the rest of the incoming edges to $rep(c)$.
They may come from the $reps$ of children of $c$ that are not in $supported_{r_1}(B)$, from layer $\ell-1$ neurons that are bound to concepts that are not children of $c$, and from unbound layer $\ell-1$ neurons.
However, the $reps$ of children of $c$ that are not in $supported_{r_1}(B)$ do not fire, by the inductive hypothesis, and the unbound neurons do not fire, by \autoref{lem: unbound-neurons}.
So that leaves us to consider the layer $\ell-1$ neurons that are bound to concepts in $C$ that are not children of $c$.  
There are at most $k^{\lmax} + 1$ such neurons.
Since the weights of the edges connecting them to $rep(c)$ are at most $\frac{1}{k^{\lmax + b}}$, the total incoming potential for $rep(c)$ in round $t+\ell$ is at most
\[
\frac{{\lfloor r_1 k \rfloor}}{\sqrt{k}} + \frac{k^{\lmax+1}}{k^{\lmax+b}} 
= \frac{{\lfloor r_1 k \rfloor}}{\sqrt{k}} + \frac{1}{k^{b-1}}.
\]
As in the base case, this is strictly less than $\tau$.  Therefore, $rep(c)$ does not fire at time $t+level(c)$.
\end{proof}

\section{Analysis of Noisy Learning}\label{sec:noisyanalysis}

Here we present our analysis for the noisy learning algorithm in
\autoref{sec:noisy}.
In \autoref{lem:noisystructure}, we describe how incoming weights change for a particular neuron when it is noisy-shown. The proof can be found in \autoref{sec:towards}.
Once we understand the weight changes of one neuron, we are able to use essentially the same invariants as in the noise-free case (\autoref{lem: main-noise-free}), describing how neurons get bound to concepts, when neuron firing occurs, and how weights change, during the time when the network is learning.
In \autoref{sec:analysisnoise}, we put everything together to prove \autoref{thm:noisylearning}.

We start by giving a slightly more detailed proof overview than the one in \autoref{sec:noisyproofidearough}.

\subsection{Proof Overview}

The overall proof 
of \autoref{thm:noisylearning} is at its core similar to the proof of \autoref{thm:noisefreelearning} presented in \autoref{sec:noisefreeanalysis}.
The main difference is that the weights of the neurons after learning are slightly different:
following the notation of \autoref{lem:wincr}, \autoref{lem:invariants} and \autoref{lem:structure}, we show that, for every $i
\in F$, the weight will eventually approximate  \[ \bar{w} =\frac{1}{\sqrt{ p k + 1-p   } } ,\]
and for every $i\not \in F,$ the weight will eventually be in the interval $[0,1/k^{2\lmax}]$. 
Note that, in this section, we set the parameter  $b$, governing the desired decrease of unrelated weights, to be $b = \lmax$.
Also note that we can recover the noise-free case by setting $p=1$.\footnote{In this case the probabilistic guarantees become deterministic guarantees.}

The main difficulty in the noisy case is to establish  a noisy version of \autoref{lem:structure}, which we do in
\autoref{lem:noisystructure}. Then, proving the main theorem is analogous to the noise-free case. 
This is because the behavior of this network is the same as that of the noise-free algorithm, except for how the weights of individual neurons are updated.
Nonetheless, the same arguments as in the proof \autoref{lem: main-noise-free} still hold. 
Therefore, the core of this section is to prove \autoref{lem:noisystructure}.
Due to the noise, main structural properties of the noise-free case, such as weights of neurons in $F$ changing monotonically, do not hold anymore.
To make matters worse, we cannot simply use Chernoff bounds and assume the worst-case distribution of the weight changes, since assuming worst-case in each round prevents the weights from converging.  Instead, we use a fine-grained potential analysis.

We first bound the worst-case change of any weight $w_i$ during a period of $T$ rounds (\autoref{lem:Eprimeholds}), 
assuming that the weight at the beginning of the period, $w_i(t)$, is in the interval $[ \frac{\sqrt{p}}{4k}, \frac{4}{\sqrt{p}}] $.
Namely, we show that for some small $\delta_1$ (defined in \autoref{sec:noisy-weight-change-individual}), we have
$(1-\delta_1) w_i(t) \leq w_i(t+T) \leq (1+\delta_1)w_i(t)$.
 We later show that this assumption holds w.h.p. throughout the first $n^6$ rounds.
It turns out that the way an individual weight changes depends strongly on the other weights in $F$ and on the neurons of the previous layer that fire. More precisely, it depends on $z(t)$, which can change dramatically between rounds, rendering the analysis non-trivial.
In order to show that the weights converge to $\bar{w}$, we use the  potential function $\psi(\cdot)$.
For any time $t$, let $w_{min}(t)$ and $w_{max}(t)$ be the minimum and maximum weight, respectively, among $\{ w_{i}(t) ~|~ i \in F\}$. 
%
Let 
\[\psi(t) = \max \left\{  \frac{w_{max}(t)}{\bar{w}}, \frac{\bar{w}}{w_{min}(t)}  \right\} . \]

Our goal is to show that this potential decreases quickly until it is very close to $1$.
Showing that the potential decreases is involved, since one cannot simply use a worst-case approach, due to the terms in Oja's rule being non-linear and potentially having a high variance, depending on the distribution of weights.  
Instead, we consider the terms
$\bar{w}/w_{min}(t)$  and  $w_{max}(t)/\bar{w}$ of the potential and consider four cases depending on whether these terms are small or large.

First, if the term $\bar{w}/w_{min}(t)$ is large and the term $w_{max}(t)/\bar{w}$ is small, then the minimum weight $w_{min}$ increases and since the maximum weight $w_{max}$ increases by at most a factor of $(1+\delta)$, the potential decreases.
The second case, where the term $w_{max}(t)/\bar{w}$ is large and the term  $\bar{w}/w_{min}(t)$ is small, can be bounded analogously.
Finally, if $\bar{w}/w_{min}(t)$ and $w_{max}(t)/\bar{w}$ are both large and close to each other, then we show that both terms decrease.
Note that if both terms are small, then the potential is small and we are done.

For example, to prove the first case, we first show that, for every $i \in F$ with $w_i(t) \geq (1+2\delta_1) w_{min}$, we have 
$w_{i}(t+T) \geq  (1+\delta/2) w_{min}$, using the previously established bounds. 
 As mentioned before, in order to prove that any such neuron $i^*$ increases its weight, we cannot use worst-case bounds. Instead, we  carefully use the randomness over the input vector $x$.
To this end we define, for every $t'\geq 0$,

\[ X(t')= z(t+t')\cdot \left( x_{i^*}(t+ t')-z(t+ t') \cdot w_{i^*}(t+t') \right) \]
and
\begin{equation}\label{eq:S} S= \sum_{t' = 1}^T X(t'). 
\end{equation}
Based on these terms we construct a Doob martingale (\autoref{lem:doopstep}), which allows us to  get asymptotically almost tight bounds on $S$, 
To do this, we use the Azuma-Hoeffding inequality (\autoref{pro:hoeff}).
Putting everything together, we see that $\psi(\cdot)$ decreases. 
This then allows us to prove  \autoref{thm:noisylearning}.

\subsection{Convergence of the Weights}
\label{sec:noisy-weight-change-individual}

We use the following assumptions about the various parameters:
\begin{enumerate}
\item
$\delta = \bar{w}(r_2-r_1)/50$, 
\item
$\delta_1 = \frac{\delta p \bar{w}}{20}$,
\item
$T=\frac{7 \log (|C|n)}{100 p^3 \delta_1^2}$,
\item 
The learning rate $\eta = \frac{\delta_1^3}{ 64 T k^2 p }$. 
\item
The firing threshold $\tau = r_2 k (\bar{w}-2\delta)$
\item $b = \lmax$.
\end{enumerate}

The following lemma is the noisy counterpart to \autoref{lem:structure}.
\begin{lemma}[Learning Properties, Noisy Case]
\label{lem:noisystructure}
Let $F \subseteq \{1,\ldots,n\}$ with $|F| = k$.  
Let $\varepsilon \in (0,1]$.

\snote{This seems to be wrong.  $b$ doesn't appear in the statement of Lemma 8.1.  I suspect that you intended to fix $b$ as you promised to in the proof overview section, Section 8.1.  However, I think that the place to actually fix $b$ would be just before Lemma 8.1, in the list where you formally fix the values of the other parameters.}
\snote{That b was just an artifact. Removed.}
\snote{But shouldn't we be fixing $b$ in the list of parameters before the lemma?  I know, we mentioned that earlier, but here seems to be where we are fixing things formally, and it's convenient for reference, to have all the restrictions in one convenient place.}
Let $\sigma = c' \frac{k^6}{p^6\delta^3}\left(\lmax \log(k) + 
  \log(|C|n/\delta \right)$, for some large enough constant $c'$. 

Assume that:
\begin{enumerate}
    \item For every $t \geq 0$, $x_i(t) = 0$ for every $i \notin F$, and $e(t) = 1$. 
    \item All weights $w_i(0)$ are equal to $\frac{1}{k}$.
    \item $\eta $ is defined above.\footnote{This is a very precise assumption but it could be weakened, at a corresponding cost in run time.}

\end{enumerate}
Then for every $t \in [ \sigma,n^6]$, the following with high probability:
\begin{enumerate}
\item 
For any $i \in F$, we have
$w_i(t) \in [ \bar{w}-2\delta,\bar{w} + 2\delta]$.
\item 
For any $i \notin F$,  we have $w_i(t) \leq \frac{1}{k^{2\lmax}}$.

\end{enumerate} 
\end{lemma}
Proving \autoref{lem:noisystructure} is the main goal of the section and we need a series of properties to prove it. 
We give the proof in \autoref{sec:proofnoisystructure}.  
We now proceed by showing how \autoref{thm:noisylearning} follows from this lemma.

\subsection{Proof of \autoref{thm:noisylearning}, assuming \autoref{lem:noisystructure}}\label{sec:analysisnoise}

\snote{We have to learn ALL the concepts in the hierarchy, and we can't start learning $c$ until its children are fully learned.  But this suggests a union bound, for the probabilities of fully learning each concept within time $\sigma$ after its children are fully learned.}
\snote{The only thing we need is that $\mathcal{E}$ holds. Once we condition on that, everything else (each child concept learns correct for example) holds w.p. 1}


As mentioned at the beginning of this section, it suffices to consider the learning of one concept. Generalizing to  a concept hierarchy is analogous to the noise-free case (in particular the proof of \autoref{lem: main-noise-free}).

We now argue how the learning of one concept follows from  \autoref{lem:noisystructure}.
By \autoref{lem:noisystructure}, all weights in $F$ are at least $\bar{w}-2\delta$ and most $\bar{w}+2\delta$.
Hence, if $c\in supported_{r_2}(B)$, then we can show by a similar induction as in the proof of \autoref{thm:noisefreelearning} that each $rep$ fires since, the potential  is at least $
r_2 k (\bar{w}-2\delta) = \tau$, which means that the corresponding $rep$ fires.
On other other hand,  if $c \not\in supported_{r_1}(B)$, then there will be a neuron that does not fire since all weights are, by  \autoref{lem:noisystructure}, at most $\bar{w}+2\delta$.

Note that, by definition of $\delta$, 
\begin{align*} r_1(\bar{w} +2\delta)  &=  (r_2-50\delta/\bar{w})(\bar{w} +2\delta)\\
&\leq r_2 \bar{w} + 2\delta r_2 - 50 \delta  \\
&\leq r_2 \bar{w} - 2\delta r_2 - 46 \delta,
\end{align*}
since $r_2 \leq 1.$
Therefore, the potential for $rep(c)$ will be at most

\[r_1 k (\bar{w}+2\delta) + k^{\lmax}\frac{1}{k^{2\lmax} } < r_2 k\left( \bar{w} - 2\delta  - \frac{46 \delta}{r_2} \right) +\frac{1}{k} \leq r_2 k (\bar{w}-2\delta)=\tau,\]

since $ k 46 \delta = k\frac{46}{50}\bar{w} (r_2-r_1) \geq  \frac{46}{50\sqrt{k}} \geq 1/k$, due to $\bar{w}\geq 1/\sqrt{k}$, $r_2-r_1 \geq 1/k$ and $k\geq 2$.
Thus, the neuron does not fire.

\subsection{Towards \autoref{lem:noisystructure}}\label{sec:towards}

In this subsection, we define a key property $\mathcal{E}_t$ that says that the weights remain within certain multiplicative bounds, for during the interval $[t,t+T]$ rounds.
We show in \autoref{lem:Eprimeholds} that $\mathcal{E}_t$ holds with  probability $1$.
Then we assume $\mathcal{E}$ and show \autoref{lem:bounds}, which bounds the expected change of the terms in Oja's rule.
To derive bounds on the actual change we first show how the changes form a Doob-martingale (\autoref{lem:doopstep}). Using this, we are finally able to show in
 in \autoref{lem:bounds1} and \autoref{lem:bounds2} that the potential decreases.

Let $\mathcal{E}_{t}$ be the event that for every $t' \in [t, t+T]$, we have  \[
\left(1-\delta_1 \right)w_i(t)
\leq w_i(t')\leq \left(1+\delta_1 \right)w_i(t) .\]
\snote{I changed several uses of "for all" to "for every" in English descriptions.  I know, the formal logical term is "for all", but really, we are talking about individual values, so "for every" seems clearer to me.}
\snote{Hmm...E has a free variable of $t$ inside, but that is not given as a parameter of E.  I don't think you want to fix the value of $t$ throughout this subsection.  So I suggest including the parameter $t$ as an additional subscript for E.}

\begin{lemma}\label{lem:Eprimeholds}
Assume $w_i(t) \in [ \frac{\sqrt{p}}{4k}, \frac{4}{\sqrt{p}}] $.
Then, $\mathcal{E}_{t}$ holds.
\end{lemma}
\snote{The statement of Lemma 8.2 would be clearer if we had the subscript t included for E.}
\snote{There are subscripts}
\snote{OK, you are instantiating the parameter tau here as capital-T.  Yes?}\snote{Let me actually make it easier.}

\begin{proof}
Let $w_{max}(t)$ denote the maximum weight at time $t$.
We have
$w_{max}(t+1)\leq w_{max}(t)+\eta z(t) \leq 
w_{max}(t)+\eta w_{max}(t) kp.
$
Thus,
$w_{max}(t+t')\leq w_{max}(t) (1+ \eta kp)^{T}=w_{max}(t)\left(
1+ \frac{\eta k p T}{T}
\right)^T=w_{max}(t)e^x$
for $x=\eta k p T$.
Since $p\geq 1/k$, we have $x<1$, we have

\[ w_{max}(t+t') \leq w_{max}(t)e^x \leq w_{max}(t)(1+x+x^2) \leq w_{max}(t)(1+2x). \]
this completes the upper bound of $\mathcal{E}_t$
since $2 \eta k p T \leq \delta_1$.

We now consider the lower bound of $\mathcal{E}_t$.
Similarly,
if $w_{min}(t)$ denotes the minimum weight at time $t$, then

$w_{min}(t+1)\geq w_{min}(t)-\eta z^2(t) \geq 
w_{min}(t)-\eta w^2_{max}(t) k^2p^2 \geq
w_{min}(t)-\eta16k^2 p.
$
Thus
$
w_{min}(t+1)\geq
w_{min}(t)- T\eta 16 k^2 p
\geq
w_{min}(t)- T\eta 16 k^2 p \frac{1}{\sqrt{p}/(4k)}w_{min}(t)\geq
w_{min}(t)\left(1- 64 \eta T k^2 \sqrt{p}
\right)\geq w_{min}(t)(1-\delta_1)
%
,$
since $w_{min}(t) \geq \frac{\sqrt{p}}{4k}$. 
\end{proof}

\

We define the following potential function
\[\phi(t)= \sum_{i\in F} w_i(t) .\]

The following bounds the expected change of the weights.
\begin{lemma}\label{lem:bounds}
Suppose $\mathcal{E}_{t}$ holds.
Then, we have

\begin{enumerate}
\item
$
 \E{z(t+t')~|~w(t+t'), \mathcal{F}_t} = p \phi(t+t')
$
\item
$
 \E{z(t+t')^2w_{i^*}(t+t')~|~\mathcal{F}_t} \leq (1+\delta_1)^3 p \phi(t)  \left((1-p) w_{max}(t)w_{i^*}(t)   + p w_{i^*}(t) \phi(t)  \right)\\
$ \item
 
$
\E{z(t+t')^2w_{i^*}(t+t')~|~\mathcal{F}_t}\geq
(1-\delta_1)^3p \phi(t)  \left((1-p) w_{min}(t)w_{i^*}(t)   + p w_{i^*}(t) \phi(t)  \right).
$ \end{enumerate}
\end{lemma}
\begin{proof}

In the following, the randomness is over $x_i(t+t')$.
%
We have,
\begin{align*}
  \E{z(t+t')~|~w(t+t'), \mathcal{F}_t} = p\sum_{i \in F}\E{x_i(t+t')}w_i(t+t') = p\sum_{i \in F}w_i(t+t')  = p \phi(t+t').
 \end{align*}

Moreover,
\begin{align*}
 \E{z(t+t')^2~|~w(t+t'),\mathcal{F}_t} &=     \sum_{i\in F}  \left( p w_i(t+t')^2+ p^2 w_i(t+t')   \sum_{j\in F, j\neq i } w_j(t+t')    \right)\\
 &= \sum_{i\in F}  \left( p w_i(t+t')^2-p^2 w_i(t+t')^2+ p^2 w_i(t+t')  \phi(t+t')  \right)\\
  &= (p-p^2) \sum_{i\in F}  w_i(t+t')^2  + p^2 \phi(t+t')^2  .
 \end{align*}

We suppose $\mathcal{E}_{t}$ holds, thus in every obtainable configuration it must hold that $ (1-\delta_1)w_i(t) \leq w_i(t+t') \leq (1+\delta_1)w_i(t)$.
Therefore, $(1-\delta_1)\phi(t) \leq \phi(t+t')\leq (1+\delta_1)\phi(t)$.
Thus,

\begin{align*}
&\E{z(t+t')^2w_{i^*}(t+t')~|~\mathcal{F}_t}=\\
&\phantom{xx}=
\sum_{w'~\land~\text{$w'$ obtainable}}
\E{z(t+t')^2w_{i^*}(t+t')~|~w(t+t')=w',\mathcal{F}_t} \Pr{w(t+t')=w'} \\
&\phantom{xx}= \sum_{w'~\land~\text{$w'$ obtainable}}
w'_{i^*}(t+t')
\E{z(t+t')^2~|~w(t+t')=w',\mathcal{F}_t} \Pr{w(t+t')=w'}\\
&\phantom{xx}\leq 
(1+\delta_1)w_{i^*}(t)
\sum_{w'~\land~\text{$w'$ obtainable}}\hspace{-0.2cm}
\left( (p-p^2) \sum_{i\in F}  w'_i(t+t')^2  + p^2 \phi(t+t')^2  \right)\Pr{w(t+t')=w'}\\
 &\phantom{xx}\leq
 (1+\delta_1)^3 w_{i^*}(t) \left((p-p^2) \sum_{i\in F}  w_i(t)^2  + p^2 \phi(t)^2  \right)
\\
  &\phantom{xx}\leq w_{i^*}(t)(1+\delta_1)^3  \left((p-p^2) w_{max}(t) \phi(t)  + p^2 \phi(t)^2 \right)\\
    &\phantom{xx}\leq (1+\delta_1)^3 p \phi(t)  \left((1-p) w_{max}(t)w_{i^*}(t)   + p w_{i^*}(t) \phi(t)  \right).
 \end{align*}

Similarly,

\begin{align*}
\E{z(t+t')^2w_{i^*}(t+t')~|~\mathcal{F}_t}&\geq
(1-\delta_1)^3p \phi(t)  \left((1-p) w_{min}(t)w_{i^*}(t)   + p w_{i^*}(t) \phi(t)  \right).
\end{align*}

\end{proof}

In the following, we define a sequence of random variables $Y_1, Y_2, \dots$ and show it forms a Doob martingale.

\begin{lemma}\label{lem:doopstep}
Fix neuron $i^*$.
Let $X_i$ be the random choices of the $pk$ children that fire in round $i$ (in the definition of the noisy learning).
Recall that $S= \sum_{t' \leq T} z(t+t')\cdot \left( X_{i^*}(t+ t')-z(t+ t') \cdot w_{i^*}(t+t') \right).$
Let $Y_i= \E{S~|~X_i,\dots, X_1}$.
Then the following holds

\begin{enumerate}
\item The sequence
 $Y_0, Y_1, \dots, Y_T$ is a (Doob) martingale with respect to the sequence
$X_0,X_1, \dots X_T$. 
\item For all $i$, $|Y_i-Y_{i+1}| \leq 8k^2 \sqrt{p}$.
\item
 $S=\E{S~|~X_T,\dots, X_1}=Y_T$. 
\end{enumerate}

\end{lemma}

\begin{proof}
For the first part, we have, using the tower rule,
\[\E{Y_i ~|~X_{i-1},\dots,X_1} =  \E{\E{S~|~X_i,\dots, X_1} ~|~X_{i-1},\dots,X_1}=\E{S~|~X_{i-1},\dots,X_1}=Y_{i-1}.  \]

For the second part, note that $w_i  \leq 2/\sqrt{p}$. Thus,
$|Y_i-Y_{i+1}| \leq z^2_{t+i} w_{i^*} \leq  k^2p^2 2^3/\sqrt{p}^3 $.

The third part follows trivially.

\end{proof}

Let  \[ \delta_2 = \left(k \frac{\sqrt{p}}{2k} \right)p^2 \left(\frac{20\delta_1}{p}\right) = 10 p^{3/2} \delta_1\]

The following lemma shows that if the potential is large due to $w_{min}$ being small, then the weight of the smallest neurons increases. 

\begin{lemma}\label{lem:bounds1}
Suppose $\mathcal{E}_{t}$ holds.
Consider the neurons $i^*$ with  $w_{i^*}(t) \in [w_{min},(1+2\delta_1) w_{min}]$
and
$\bar{w}-w_{i^*}(t) \geq  \delta.$ 
Assume
 \begin{align}\label{eq:klaskd}
 \frac{\bar{w}}{w_{min}(t)} \geq (1-2\delta_1)  \frac{w_{max}(t)}{\bar{w}}  .
 \end{align}
Then, with probability at least $1-1/n^6$,
\[ w_{i^*}(t+T) \geq w_{i^*}(t) + T \eta\delta_2/2
\]
\end{lemma}

\begin{proof}

By the second part of \autoref{lem:bounds}, for $t'\leq T$
\begin{align*}
 \E{z(t+t')^2w_{i^*}(t+t')} &\leq (1+\delta_1)^3 p \phi(t)  \left((1-p) w_{max}(t)w_{i^*}(t)   + p w_{i^*}(t) \phi(t)  \right).
 \end{align*}
We now bound the  terms in the parentheses.
First note that
\[
w_{i^*}(t) w_{max}(t)  \leq   (1+2\delta_1) w_{min}(t) w_{max}(t)   \leq \frac{1+2\delta_1}{1-2\delta_1} \bar{w}^2 \leq (1+4.5\delta_1) \bar{w}^2, \]
since $\delta_1 \in [0,1/18]$.
Furthermore, for $\delta_1 \in [0,1/9]$ we have
$(1+4.5\delta_1)(1+\delta_1) \leq (1+6\delta_1)$. Thus,
\begin{align*}
w_{i^*}(t) \phi(t) &\leq  (k-1)(1+\delta_1) w_{i^*}(t) w_{max} + (1+\delta_1)w_{i^*}(t)  w_{i^*}(t)\\
&\leq  (k-1)(1+\delta_1)(1+4.5\delta_1) \bar{w}^2 + (1+\delta_1)w_{i^*}(t)  w_{i^*}(t)\\
&\leq    (1+6\delta_1) \left( (k-1) \bar{w}^2 + w_{i^*}(t)^2 \right)\\
&=   (1+6\delta_1) \left( k \bar{w}^2 + w_{i^*}(t)^2 - \bar{w}^2 \right).
\end{align*}
Note that
$(1-p)    \bar{w}^2 + p k \bar{w}^2 = 1 $.
Thus,
\begin{align*}
(1-p) w_{max}(t)w_{i^*}(t)   + p w_{i^*}(t) \phi(t)   &\leq  (1+6\delta_1)   \left( (1-p)    \bar{w}^2 + p k \bar{w}^2 + p(w_{i^*}(t)^2 - \bar{w}^2)      \right)   \\
&= (1+6\delta_1) \left( 1 -  p(\bar{w}^2 -w_{i^*}(t)^2)      \right).
\end{align*}

Therefore,
\begin{align*}
\E{z(t+t')^2w_{i^*}(t+t')} \leq (1+10\delta_1) p \phi(t)  \left( 1 -  p(  \bar{w}^2 -w_{i^*}(t)^2  )  \right),
\end{align*}
where we used that
$(1+6x)(1+x)^3 \leq (1+10x)$ for $x\leq 0.045$.

Note that 
\begin{align}\label{eq:difftotarget}
\bar{w}^2 -w_{i^*}(t)^2\geq \bar{w}^2 -w_{i^*}(t)\bar{w}=\bar{w}(\bar{w}-w_{i^*}(t)) \geq \bar{w} \delta= \bar{w} \frac{20}{\bar{w}p}\delta_1 .
\end{align}

Finally, using the definition of $S$ (\autoref{eq:S}) and combining the above with 
 the first part of \autoref{lem:bounds},
\begin{align*}
\E{S} &\geq  T\left( \E{z(t+t')} - \E{z(t+t')^2w_{i^*}(t+t')}   \right) \\
&\geq  T \phi(t) p\left( 1 - (1+10\delta_1)  \left( 1 -  p(\bar{w}^2 -w_{i^*}(t)^2 )  \right)  \right)\\
&\geq  T \phi(t) p^2 \frac{\bar{w}^2 - w_{i^*}(t)^2 }{2} ,
\end{align*}
where we used that
$1-(1+z)(1-x) = 1-(1-x+z-zx) = x-z+zx \geq x/2 $ for  $z\leq x/2$.
%
%
We define the sequence $Y_{1}, Y_2, \dots$ of variabels as defined in \autoref{lem:doopstep}.
By \autoref{lem:doopstep}, this sequence is a  Doob martingale.
Thus, we can apply
\autoref{pro:hoeff} to the Doob martingale $Y_T,Y_{T-1},\dots, Y_1$ with  $|Y_i-Y_{i+1}| \leq \delta_3$ for $\delta_3=8k^2 \sqrt{p}$. 

We derive
using the lower bounds on the weights and \autoref{eq:difftotarget}.
\begin{align*}
		\Pr{|S - \E{S}| \geq \frac{\E{S}}{2} } &\leq  2\exp\left(-\frac{2\left(\frac{\E{S}}{2}\right)^2}{T\delta_3^2}\right)
		\leq 2\exp\left(-\frac{2 \left(T \phi(t) p^2 \frac{ \bar{w}^2 - w_{i^*}(t)^2 }{2} \right)^2   }{T\delta_3^2 }\right)\\
			&\leq 2\exp\left(-\frac{ T \left(\phi(t) p^2 \left( \bar{w}^2- w_{i^*}(t)^2\right) \right)^2   }{4\delta_3^2}\right)
			\leq 2\exp\left( -7 \lmax \log (|C|n)\right) \leq   \frac{1}{|C| n^{6}},
	\end{align*}
where the last inequality follows from
\begin{align*}
T \left(\phi(t) p^2 \left( \bar{w}^2- w_{i^*}(t)^2\right)\right)^2 
&\geq T \delta_2^2
= 
100 T p^3 \delta_1^2
= 7\log (|C| n).
\end{align*}

Thus 
\begin{align*}
w_{i^*}(t+T) &\geq w_{i^*}(t) + \eta S \geq
 w_{i^*}(t) + \eta \E{S}/2 \geq   w_{i^*}(t) + 
T \eta \delta_2/2
\end{align*}
\end{proof}

The following lemma is analogous to the previous one, with the difference that we analyse the case where $\psi$ is dominated by large weights (rather than small) and show that these large weights decrease. 

\begin{lemma}\label{lem:bounds2}
Suppose $\mathcal{E}_{t}$ holds.
Consider the neurons $i^*$ with  $w_{i^*}(t) \in [w_{max}(1-2\delta_1), w_{max}]$ and
 $w_{i^*}(t)-\bar{w} \geq  \delta. $
Assume
\begin{align}\label{eq:klaskd2}
 \frac{w_{max}(t)}{\bar{w}}  \geq (1-2\delta_1) 
  \frac{\bar{w}}{w_{min}(t)} 
 \end{align}
Then, with probability at least $1-1/n^6$,
\[ w_{i^*}(t+T) \leq w_{i^*}(t) - T\eta \delta_2/2
\]

\end{lemma}
\begin{proof}

We have for all $i \in F$ with $w_i(t) \geq (1+2\delta_1) w_{min}$, we have 
 $w_{i}(t+T) \geq  (1+\delta_1/2) w_{min}$, since each weight can only decrease by a factor of $(1-\delta_1)$ and since
 $(1+2\delta_1)(1-\delta_1)=1+\delta_1-2\delta_1\geq (1+\delta/2)$.
 Thus, we only consider the neurons $i^*$ with  $w_{i^*}(t) \in [w_{min},(1+2\delta_1) w_{min}]$.
By the third part of \autoref{lem:bounds}, for $t'\leq T$
\begin{align*}
 \E{z(t+t')^2w_{i^*}(t+t')} &\geq (1-\delta_1)^3 p \phi(t)  \left((1-p) w_{min}(t)w_{i^*}(t)   + p w_{i^*}(t) \phi(t)  \right).
 \end{align*}
We now bound the  terms in the parentheses.
First note that
\[w_{i^*}(t) w_{min}(t)  \geq   (1-2\delta_1) w_{min}(t) w_{max}(t)   \ge (1-2\delta_1)^2\bar{w}^2 \geq (1-4\delta_1) \bar{w}^2, \]
since $\delta_1 \geq 0$.

Thus,
\begin{align*}
(1-p) w_{min}(t)w_{i^*}(t)   + p w_{i^*}(t) \phi(t)   &\geq  (1-4\delta_1)   \left( (1-p)    \bar{w}^2 + p k \bar{w}^2 + p(w_{i^*}(t)^2 - \bar{w}^2)      \right)   \\
&= (1-4\delta_1) \left( 1 -  p(\bar{w}^2 -w_{i^*}(t)^2)      \right)
\end{align*}

Therefore,
\begin{align*}
\E{z(t+t')^2w_{i^*}(t+t')} \geq (1-10\delta_1) p \phi(t)  \left( 1 -  p(  \bar{w}^2 -w_{i^*}(t)^2  )  \right),
\end{align*}
where we used that
$(1-4x)(1-x)^3 \geq (1-10x)$ for $x\geq 0$.

Note that 
\begin{align}
\bar{w}(w_{i^*}(t)-\bar{w}) \geq \bar{w} \delta =\bar{w} \frac{20}{\bar{w}p}\delta_1 
\end{align}

Finally, using the definition of $S$ (\autoref{eq:S}) and combining the above with 
 the first part of \autoref{lem:bounds},
\begin{align*}
\E{S} &\leq  T\left( \E{z(t+t')} - \E{z(t+t')^2w_{i^*}(t+t')}   \right) \\
&\leq  T \phi(t) p\left( 1 - (1-10\delta_1)  \left( 1 -  p(\bar{w}^2 -w_{i^*}(t)^2 )  \right)  \right)\\
&\leq 2 T \phi(t) p^2 \bar{w}^2 - w_{i^*}(t)^2 
=- 2 T \phi(t) p^2(w_{i^*}(t)^2-\bar{w}^2)\\
&\leq 
- 2 T \phi(t) p^2\bar{w}(w_{i^*}(t)-\bar{w}),
\end{align*}
where we used that
$1-(1-z)(1-x) = 1-(1-x-z+zx) = x-z+zx \leq 2x $ for  $z\leq 1$.

This allows us to apply  \autoref{pro:hoeff} and the rest is analogous.

Thus 
\begin{align*}
w_{i^*}(t+T) &\leq w_{i^*}(t) + \eta S \leq
 w_{i^*}(t) +\eta \E{S}/2 \leq   w_{i^*}(t) - 
T \eta \delta_2/2\end{align*}
\end{proof}

We have for all $i \in F$ with $w_i(t) \geq (1+2\delta_1) w_{min}$, we have 
 $w_{i}(t+T) \geq  (1+\delta_1/2) w_{min}$, since each weight can only decrease by a factor of $(1-\delta_1)$ and since
 $(1+2\delta_1)(1-\delta_1)=1+\delta_1-2\delta_1\geq (1+\delta/2)$.

 Note that if neither \autoref{eq:klaskd} nor \autoref{eq:klaskd2} applies, then both $w_{min}(t)$ and $w_{max}(t)$ must be close to $\bar{w}$ and the claim follows easily.

\subsection{Proof of \autoref{lem:noisystructure}}\label{sec:proofnoisystructure}

We argue by induction on $j$, that
$\psi(j\cdot T) \leq  \max\left(\psi(0) - jT \eta \delta_2 /2, 
\bar{w} +2\delta \right)$
with probability at least $1-j/(|C|n^6)$. The base case is trivial.
Assume the claim holds up to $j-1$.
We have

 $w_i((j-1)T) \in [ \frac{\sqrt{p}}{4k}, \frac{4}{\sqrt{p}}] $. Therefore,  by \autoref{lem:Eprimeholds} $\mathcal{E}_{(j-1)T,T}$ holds.
This allows us to apply \autoref{lem:bounds1} and \autoref{lem:bounds2}.

Consider the following equations
\begin{align}\label{eq:klaskd3}
 \frac{\bar{w}}{w_{min}(t)} \geq (1-2\delta_1)  \frac{w_{max}(t)}{\bar{w}}  .
 \end{align}

\begin{align}\label{eq:klaskd4}
 \frac{w_{max}(t)}{\bar{w}}  \geq (1-2\delta_1) 
  \frac{\bar{w}}{w_{min}(t)} 
 \end{align}

We consider four cases based on whether or not the two equations \autoref{eq:klaskd3} and \autoref{eq:klaskd4} hold.
 In the first case
 \autoref{eq:klaskd3} holds and \autoref{eq:klaskd4} does not.
 In this case we can bound the drop of $\psi()$ by considering the the increase of $w_{min}()$ and we can disregard the increase of $w_{max}()$, since even if it increases  by a factor of $(1+\delta_1$), we have
 \[ \frac{w_{max}(jT)}{\bar{w}} \leq (1+\delta_1)\frac{w_{max}((j-1)T)}{\bar{w}} \leq (1+\delta_1)(1-2\delta_1) \frac{\bar{w}}{w_{min}((j-1)T)} \leq (1-\delta_1) \frac{\bar{w}}{w_{min}((j-1)T)}.
 \]

 In the second case
 \autoref{eq:klaskd4} holds and \autoref{eq:klaskd3} does not.
 This case is analogous to the first case.
 
 In the third case \autoref{eq:klaskd3} and \autoref{eq:klaskd4} hold. Here, one can show that both the minimum weight increases, and the maximum weight decreases.
 
 In the fourth case, none of the equations hold. This yields a contradiction 
 \[ 
  \frac{\bar{w}}{w_{min}(t)} <  (1-2\delta_1)  \frac{w_{max}(t)}{\bar{w}} < (1-2\delta_1)^2 
  \frac{\bar{w}}{w_{min}(t)}. 
  \]
 Thus we can disregard this case.
 
W.l.o.g.  we assume the first case holds.

Consider the neurons $i^*$ with  $w_{i^*}(t) \in [w_{min},(1+2\delta_1) w_{min}]$
and
$\bar{w}-w_{i^*}(t) \geq  \delta.$ 
Then, by \autoref{lem:bounds1}, with probability at least $1-1/n^6$,
\[ w_{i^*}(t+T) \geq w_{i^*}(t) + T \eta \delta_2/2 \geq  w_{i^*}(t) + w_{i^*}(t) \frac{T \eta \delta_2}{2 (4\sqrt{p}))}.  
\]

Note that in the analogous cases two and three we have
for any neurons $i^*$ with  $w_{i^*}(t) \in [w_{max}(1-2\delta_1), w_{max}]$ 
that
\[ w_{i^*}(t+T) \leq  w_{i^*}(t) - T \eta \delta_2/2 \leq  w_{i^*}(t) - w _{i^*}(t) \frac{T \eta \delta_2}{2 (4\sqrt{p}))} . \]

Let $\delta_4= T\eta\delta_2 / (8\sqrt{p})$.
Thus, either way

\[\psi(jT) \leq (1-\delta_4) \psi((j-1)T). \]

Using the fact that $\log(1+x)\geq 2x$ for $x \in (-1/2,0)$, we get that after \[j^* = \log_{1-\delta_4}(\delta/\psi(0))= \frac{\log(\delta/\psi(0))}{\log(1-\delta_4)} \leq \frac{\log(\delta/\psi(0)) }{-2\delta_4} = \frac{\log(\psi(0)/\delta) }{2\delta_4}
\] intervals of length $T$ the $\psi()$ is within an error of at most $2\delta$ and stays there by assumption for $n^6$ rounds.
Thus the total number of rounds is $Tj^*$.
The bound from the claim follows by observing that term $\eta T /\delta_4$ is a small polynomial in $p$ and $w$ and $\delta$.

Finally, we consider the time  required for weights $i \not \in F$ to decreases below $k^{-2\lmax}$.
After the weights in $F$ are close to there target, we have
that $z(t)\geq p k \bar{w}/2.$
Thus at this point, the weights decrease  changes as follows every round

 \[w_i(t)= w_i(t-1)(1 - \eta z(t-1)^2 )\geq w_i(t-1)(1 - \eta p^2 k^2 \bar{w}^2/4)) .
 \]
 Thus, the potential halves every $20/(\eta p^2 k^2 \bar{w}^2)$ rounds.
Since the potential only needs to drop by a factor of $k^{2\lmax}$, the bound follows.




\iffuture
\section{Extension: Overlap}\label{sec:overlap}

We would like to extend the work to allow limited overlap.  This section contains some notes collected from various places about this.  Other useful ideas may appear in a separate note overlap.tex.  

\paragraph{Tradeoff between tolerated noise and amount of overlap:}
It seems that there is some kind of tradeoff between the amount of noise that can be corrected and the amount of overlap that can be handled.  Pin this down.  Probably this can be stated in terms of the recognition problem, without learning.

Consider first the simple case of two concepts $c$ and $c'$, at the same level.
Then total presentation of the children of $c'$ should not trigger firing for $c$.
Let $x$ be the max overlap of these two sets of children.
In terms of concepts only (not networks), we need simply that $x < r_1 k$; here $r_1 k$ is the threshold below which the network should not recognize $c$.

In terms of the network, we can consider the situation after learning, so the learned edges all have weights near $1/sqrt(k)$.
Then firing of the neurons for $c'$ results in at most $x/sqrt(k)$ total weight incoming to $rep(c)$, which should be less than $r_1 k / sqrt(k)$ total weight.
That should be less than the threshold required for $c$ to fire.  That is consistent with our threshold setting.

But this was just a very simple case, just two concepts with overlapping sets of children.  In general, we would need to avoid overlapping among all the concepts at the same level that might be presented at the same time.  
Thus, the children of one concept cannot have overlap of $r_1 k$ with the union of the children of all the other concepts at the same level that might be presented (as part of presenting the same concept hierarchy).  It remains to define this non-overlapping condition.  It will involve multi-levels of the hierarchy.

\paragraph{Limiting the amount of overlap:}
Here is a candidate assumption for limiting the amount of overlap.
Let $c \in C_\ell$, where $1 \leq \ell \leq \lmax$.
Let $C' = \bigcup_{c' \in (C_{\ell} - \{ c \})} children(c')$ be the union of the sets of children of all the other concepts in $C_\ell$.
Then $|children(c) \cap C'| < r_1 k$.

\paragraph{Need for a new mechanism:}
To handle overlap, I think we will need a new mechanism, something that can be used to prevent a (considerably) different concept $c'$ from becoming bound to the same neuron that is already bound to a concept $c$.  In the presence of interleaving, this mechanism has to "kick in" as soon as the neuron first becomes engagesd in learning for $c$.  And it also has to work after the first concept $c$ is fully learned.

I don't think that comparing the potentials will be enough.  Consider the case where concept $c$ is fully learned, and there is only one child in the intersection of children(c) and children(c').  Then when $c'$ is first shown, $rep(c)$ will get approximately 1/sqrt(k) of total incoming potential.  But a new neuron will have only $k/k^{\lmax}$.  Can we assume this mechanism is built into the WTA somehow, i.e., that the WTA remembers when it has selected something already?  

Another idea:  Suppose that our model allowed a neuron to have two different kinds of incoming "potential":  one the ordinary potential, and the other just the number of incoming neighbors that fire (this can be viewed as potential, assuming the weights are all equal to 1).  Then we could have a WTA-style mechanism that chooses the highest potential neuron, \emph{from among those that have enough children firing}.  That seems like a simple way to rule out the bad cases of a few incoming neurons with high weights.  I bet this sort of think is implementable in our model, might even be bio-plausible, I don't know.

Slight correction, based on a conversation with Brabeeba:  The new potential can't just count the number of incoming neighbors that fire, since we have all-to-all connectivity, so the layer 0 neurons that fire for the new concept $c'$ will all be incoming neighbors of $rep(c)$, always leading to a count of $k$.  To correct for this, we can use a trick that others (Ila Fiete, e.g.) have used:  ignore all the firing neighbors whose weights are "too small", e.g., less than $\frac{1}{k^{\lmax + b}}$ for some suitable $b$  The idea is that once learning is complete, the irrelevant weights should have been reduced to near zero, so they will be ignored from then on by the WTA.  Note that this should give us a particular target for the decreasing weights.

Frederik also had an idea, about allowing neurons to remember whether their weights have changed, or something like that.

Brabeeba suggested that we might delay the invocation of WTA for a new learning instance until some weights have had a chance to increase (enough so that the old rep neuron is no longer the strongest one).  I don't see how that would work.

\nnote{
This is an old note.  I don't think it applies---it seems to be concerned about the case where a completely arbitrary $k$ inputs might arrive, and they can be enough to trigger both concepts' reps to fire.  But we have been assuming that the inputs aren't arbitrary, but are actual (parts of) particular concepts.

What I see now is just that, to avoid confusion, the number $x$ of overlapping children between two concepts should satisfy $x < r_1 k$.  Which does make sense for very small $r_1$.  Of course, we will want to extend this to more complicated types of overlap.

Old note starts here:

Consider, for example, two items $a_1$ and $a_2$ at level $1$ in the
data space, each with exactly $k$ children.  How many children might
they be allowed to have in common?
The problem is to avoid confusion, that is, a situation in which both
$a_1$ and $a_1$ might be recognized.

Let's consider the case where only $k$ inputs are presented.
So then the problem would be if the sets $children(a_1)$ and
$children(a_2)$ shared at least $r_1 k$ members.   If this doesn't
happen, then $k$ inputs aren't capable of supporting both $rep(a_1)$ and
$rep(a_2)$.

This seems to make sense only if $r_1$ is greater than $1/2$.
Let's say $r_1 = 1/2 +\varepsilon$.
To have confusion, it must be that $x$, the number in the overlap of
the two sets of children, should satisfy $2 r_1 k - x \leq k$.
So to avoid confusion, we need the opposite, that $2 r_1 k - x > k$, or
$x < (2 r_1 - 1)k = 2\varepsilon k$.

For example, if $r_2 = .7$, and $r_1 = .55$ then $\varepsilon = .05$ and
the overlap $x$ must be less than $.1 k$.
}
\else
\fi

\iffuture

\section{Sparser Connections}\label{sec:sparse}
It is easy to see that we do not require each neuron on layer $\ell$ to be connected to all neurons on layer $\ell-1$.
Our results and analysis also work with minor modifications for probabilistic connections, where each edge is present independently w.p. $p \in (0,1]$ for $p$ large enough. What is important is that there is at least one neuron per layer that that is connected to all $k$ children that belong to the concept that is being learned. In the most trivial version, we assume  we have $n$ neurons per layer, we require that $(1-p^k)^n$ is small enough so that we can take union bound over all sub-concepts. This is trivially the case for $p=\Omega(1)$ and even for smaller levels. Furthermore, it is not actually necessary for the neuron on layer $\ell$ to have connection to all $k$ children, for example if $r_2-r_1$ is sufficiently large and the neuron is connected to $r_2$ of the children, then that this suffices.

\FloatBarrier
\section{Guaranteeing the Winner-Take-All Assumption}
\label{sec:WTA}

In order to guarantee these assumption, we will generalize 
our model in various ways:
First, we introduce firing strengths that are  in $[0,1]$ (as opposed to binary values).

Second, we allow a second type of firing functions: the identity function. Thus we have two firing functions, the identity function and the threshold function.

The firing strength of a neuron is determined by the firing function $f(\cdot)$ and the potential, i.e., 
$y^{(u)}(t) = f(p^{(u)}(t-1))$.
We will use two functions in this paper
\begin{enumerate}
\item the \emph{identity function} $f(p^{(u)}) = p^{(u)}$, and
    \item the \emph{threshold function} $ f(p^{(u)}) =  \begin{cases}
1 & \text{if $p^{(u)} \geq \tau$}, \\
0 & \text{otherwise,}
\end{cases}$
 where $\tau$ is some threshold parameter.
\end{enumerate}

Third, we extend the feed-forward network by introducing some feed-backward connections that will be used to update the weights.

Fourth,
in order to allow for learning, we need to describe the timing.
As before, we assume synchronous rounds.
Fix a round $t$.
When a neuron receives the input $z$ in that round, we label it for convenience $z^{(u)}(t-1)$; as we will see in the network depicted in \autoref{fig:WTAnetwork}, the value $z$ received is in fact the one calculated in round $t-1$.
After receiving  $z^{(u)}(t-1)$, the neuron updates its weights according to Oja's rule. For this we assume that the neuron
has access to $x^{(u)}(t-1)$, i.e., the output vector of the neurons on the layer below in round $t-1$.
Observe that   $z^{(u)}(t-1)=0$ means that the weights remain unchanged, due to the definiton of Oja's rule.



A main building block we use is a Winner-Take-All Module (WTAM), that we will use as a black box.  
\begin{itemize}
\item Input is an $n$ dimensional  vector $\mathbf{z}= z_1,\dots z_n$ in $[0,1]^n$.
\item Output is an $n$ dimensional vector $\mathbf{z'} = z'_1,\dots z'_n$ in $[0,1]^n$ where $z'_i=z_i$, with $i=\arg\max_i\{ z_1,\dots z_n\}$ and for all $j\neq i$ we have $z'_j = 0$.
\end{itemize}
The above is an extension of \cite{CamITCS} from binary values to rates, we leave the exact implementation of this as an open problem.
Depending on the implementation,
some of the neurons of the WTAM may include inhibitory neurons.
We assume here that the WTAM computes the output in only one round.

Putting everything together yields \autoref{as:WTA}.

\begin{figure}[ht!]
\centering

\end{figure}

\begin{figure}
\centering
  \includegraphics[page=2,trim={0 8cm 12cm 0},clip,width=0.5\textwidth]{plan}
  \captionof{figure}{A neuron with activation function $f(\cdot)$ and three feed-forward synaptic inputs (black), and one feed-backward input for learning (blue).}
\label{fig:one_neuron}
\end{figure}

\begin{figure}
  \centering
 \includegraphics[page=1,width=0.89\textwidth]{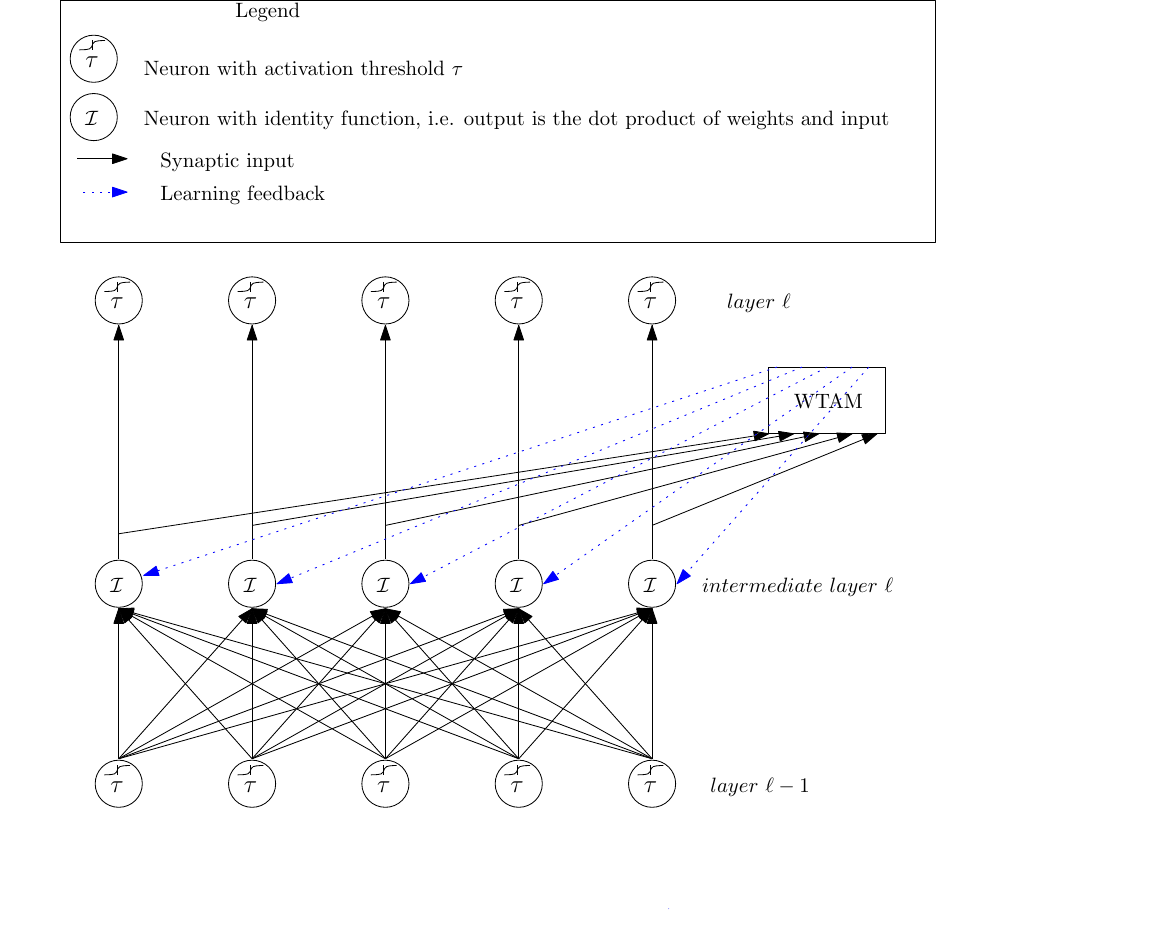}
  \caption{Model of the network with learning.}
  \label{fig:WTAnetwork}
\end{figure}

We are ready to describe the complete learning network, which is depicted in  \autoref{fig:WTAnetwork}.
The idea is that we take an arbitrary network $\cal{N}$ using the Winner-Take-All assumption and turn it into a network $\cal{N}'$ that ca be implemented in a model that requires only the simple adjustments described above.
For every layer $\ell\geq 1$, we add in $\cal{N}'$ an intermediate layer and a WTAM. The input at layer $0$ propagates one layer per round forward through the network.
The neurons on intermediate layer $\ell$ receive the input in round $2\ell -1$. In round $2\ell$, the neuron that had the largest potential on intermediate layer $\ell$ receives a learning feedback $z>0$. Again, if the learning feedback is zero, i.e, $z=0$, then the weights remain unchanged.

In order to make sure that only the neuron on the right level learn, we can simply count the number of input neurons firing and only activated the WTAM at layer $\ell$ if the number of ones in input is exactly $k^\ell$ (this can be done by using inhibitory neurons in addition with auxiliary neurons).

\else
\if


\iffuture
\snote{

\section{Thoughts on single-layer learning}

This is not edited from before so must be wildly inconsistent in
notation etc.  But it might have some useful ideas.

Restrict my data space to levels $0$ and $1$.  As usual, each item $a$
at level $1$ has a set of children, $children(a)$, at level $0$.  The
various sets $children(a)$ may overlap,  but we will probably want
sufficient separation (symmetric difference of sets, or Hamming
distance between binary representations).

Define some parameters:
Assume the data space has $n = n_0$ level $0$ items, and $n_1$ level $1$ items.
Assume each level $1$ item has exactly $m$ children, chosen from among
the level $0$ items.

Assume the network has $n = n_0$ input neurons, and $2k << n_1$ output
neurons.  It may also have some internal neurons.
We assume a 1-1 correspondence between the level $0$ data items and
the input neurons.
For a level $0$ item $a$, let $rep(a)$ denote its corresponding
(representing) input neuron.

We assume that a sequence of $k$ sets $B$ of the form
$children(a)$ are presented, each for sufficiently long (based on the
time needed to perform learning).
I think we want to assume sufficient separation between all the
presented inputs.

At the end, we want to have each such set `bound' to some unique
output neuron.  This means that a later presentation of any such set should
lead to its corresponding output neuron firing.
Probably the same for something `close' to such an input set.
And also, inputs that are `far' from such an input set should not
cause its output to fire.

Connections:  As above, consider both static and random.  First
consider static.  We will pin down the precise properties we need for
the connections.  Then in a second stage we will use JL-style random
projection and try to ensure the same properties.

Properties we assume (these will have to be loosened when we introduce
randomness):
\begin{enumerate}
\item
For every $a \in A_1$, assume a nonempty subset $R(a)$ of output
neurons.  These neurons will be the candidate sets for representations
for $a$.
\item
For every $a \in A_1$, every $y \in R(a)$, and every $b \in children(a)$,
the network contains an edge $(rep(b),y)$.
That is, there are edges from all the inputs representing children of
$a$ to every candidate representation neuron for $a$.
Say, the edges initially have weight $1$.
\item
No other edges are present.
\end{enumerate}

The network will work as follows:
This may be related to our renaming algorithm, step 3.
When the first input $B^1 = children(a_1)$ is presented, it triggers
all its possible representations $R(a_1)$ to fire.  Well, that is not
so good---we want only one to fire, and suppress the others.
Use a simple WTA?
The chosen one can inhibit the others.

Then $B^1$ continues firing for a long time, and the chosen
$rep(a_1)$ also fires, but no other neurons in $R(a_1) - \{ rep(a_)
\}$ fire.
Using some Hebbian rule, this has the effect of strengthening all the
edges of the form $(rep(b),rep(a_1))$ for every $b \in B^1$.
Strengthen them enough so some good fraction of the members of $B^1$
would be enough to trigger firing.
At the same time, this would weaken all the other edges into
$rep(a_1)$, which would prevent different input sets to bind to the
same neuron $rep(a_1)$.

We need to make sure that the other sets $B^2, \ldots$ have enough
alternative neurons so that they will always have something they can
bind to, even though some of the neurons are already bound to other
input sets.

Q:  What is `enough' here?  We need another assumption about this.
worst-case?  Like, each level $1$ item needs enough so that if all its
prececessors take one of yours you will still have another one left.   That
would be $k$ alternatives, for each level $1$ item in the data space.

But randomness might actually give us something better here.
Suppose that all of the level $1$ items choose a set of $k'$ reps
independently, uniformly among all the possible size-$k'$ subsets of
the $2k$ outputs.
Here, we look for $k' < k$.
Then how big does $k'$ have to be to get high probability that
everyone has a remaining choice?
Well, since there are only $k$ input sets, in the worst-case, half of
the possible outputs remain, and if the choices are random, it seems
that the expectation of number needed to find something unused is
roughly 2.
What do we need to get high probability $1 - \delta$?

A simplified version of the problem:
Try to minimize $k'$ such that:
Given a sequence of $k$ length-$k'$ sequences of elements from a set
of size $2k$.  Each sequence consists of distinct elements.  The
different sequences are chosen independently.
Choose the first element of the first sequence.
For subsequent sequences, choose the first element of that sequence
that hasn't yet been chosen.
How big does $k'$ have to be to get probability at least $1 - \delta$
that everyone gets a choice?

A bit of calculation needed here...
Use a union bound over all the $k$ sequences?
For one sequence, try to get probability of success at least $1 - \delta/k$.
Since half the choices are always correct, we need $(1/2)^{k'} <
\delta/k$.
That is, $k'  > \log(k) + \log(1/\delta)$, or something like that.

\paragraph{Noise-tolerance:}
Let's assume that, after learning is complete, for any level $1$ item
$a$, the firing of a fraction $r$ of the reps for $children(a)$ are
enough to trigger $rep(a)$'s firing.
We can achieve this by a Hebbian learning rule, given sufficient time.

Now suppose that we want to avoid confusion, that is, a situation
where $k$ input neurons fire, and cause both $rep(a)$ and $rep(b)$ to
fire, for two distinct level $1$ neurons $a$ and $b$.
For this, we need an assumption about the separation between children sets.
So assume that $|children(a) \cap children(b)| \leq r' k$, that is, at most
$r'$ fraction can overlap.
How large does $r'$ need to be to permit confusion to happen?

For confusion to occur, we need a set of $k$ inputs firing, among
which we have at least $r k$ of the $rep(children(a))$ neurons firing
and at least $r k$ of the $rep(children(b))$ neurons firing.
These two rep sets can have at most $r' k$ neurons in common.
Thus, in order to have confusion, we would have to have $k \geq 2rk - r'k$.
This implies that $r' \geq 2r - 1$.
Thus, we claim that if $r' < 2r - 1$, then confusion can't occur, even
with noise-tolerance.

For example, suppose that $r = 2/3$, that is, $2/3$ of the inputs are
enough to make a rep of a level $1$ item fire,
Then confusion can't occur if $r' < 1/3$, that is, if the input sets
have less than $1/3$ overlap.

TBD:  Work out all details.  See how JL-style randomness can ensure
the needed properties.   Continue this with more levels.

\section{Thoughts on more elaborate representations}
\label{sec:  more-reps}

All of these results produce a single representing neuron for each
individual item.
But we might sometimes want more neurons per item, perhaps to add
resiliency, or to allow us to `pack'' representations for more items
into a network of a given size.

For example, we might consider a set of neurons of some particular
size $k$.

We might want some clustering properties, saying that `similar''
items  (according to some  measure of closeness for data items)
should have `close'' representations (according to some measure of
closeness for sets of neurons).
And `dissimilar'' inputs should have `distant'' representations.

Interesting algorithmic issues may arise here.
We might want networks such that, as they learn a representation for
one item, adjusts weights so that later presentations of dissimilar
items become more likely to yield distant representations.
This will involve some inhibition mechanism.
Ideas may arise from Lili's inhibitory protocol, or the third stage
in the renaming paper.

\noindent
\emph{Note:}
This reminds me of something I heard in a talk in 2013, from a 
researcher at Radcliffe, about how the representations of odors in the
fruit fly brain adjust themselves over time to produce better
separation.

}
\else
\fi

\section{Auxiliary Content}

The following is a slightly modified version of Theorem 5.2 in \cite{dubhashi2009concentration}, which we use in \autoref{lem:bounds1} and \autoref{lem:bounds2}.
\begin{theorem}[Azuma-Hoeffding inequality - general version~\cite{dubhashi2009concentration}] \label{pro:hoeff}	
	Let $Y_0, Y_1,\dots$ be a martingale with respect ot the sequence $X_0,X_1,\dots$. Suppose also that $Y_i$ satisfies $a_i \leq Y_i -Y_{i-1}\leq b_i$ for all $i$.
As an example, the $engaged$ flag could be used to ensure that, in any round, only one neuron in the network is prepared to learn.
	\[
		\Pr{|Y_n - Y_0| \geq t } \leq  2\exp\left(-\frac{2t^2}{\sum_{i=1}^n (b_i -
		a_i)^2}\right).
		\]
\end{theorem}

\ifarxiv

\else
\fi
\snote{Here is a reference for place cells.  I think O'Keefe is regarded as the pioneer of this research.
O'Keefe, John (1978). The Hippocampus as a Cognitive Map. ISBN 978-0198572060.
Also see the 1976 article:
https://www.sciencedirect.com/science/article/pii/0014488676900558}
\snote{I didn't know what you mean with place cells. These are the references however \cite{Okeefe78,Okeefe76}}

\snote{Here is a reference for Matt Wilson's motion coding research:
Dynamics of the hippocampal ensemble code for space
MA Wilson, BL McNaughton
Science 261 (5124), 1055-1058
}
\snote{Found it: \cite{wilson1993dynamics}}

\snote{  Maybe:
Peretz I, Zatorre RJ. Brain organization for music processing. AnnuRev Psychol2005;56:89–114.
https://www.ncbi.nlm.nih.gov/pubmed/15709930
} \snote{\cite{peretz2005brain}}

\end{document}